\documentclass[journal]{IEEEtran}

\usepackage{color,amsmath,amsthm,amstext,amsfonts,amssymb, mathrsfs,mathtools}
\usepackage{graphicx,algorithm,algorithmic}
\usepackage[algo2e,linesnumbered]{algorithm2e}
\usepackage{tikz}
\usepackage{setspace}
\usepackage{array} 
\usepackage{booktabs}  
\usepackage{bbm}
\usepackage{wasysym}
\usepackage{textcomp}
\usepackage{hyperref} 
\usepackage{caption}
\usepackage{subcaption}
\usepackage[sort,compress]{cite} 

\newtheorem{assumption}{Assumption}

\newtheorem{proposition}{Proposition}
\newtheorem{theorem}{Theorem}

\theoremstyle{remark}
\newtheorem{remark}{Remark}

\SetKwInput{KwInput}{Input}

\newcounter{l1}
\newcounter{l2}
\newcounter{l3}
\newcommand{\bdotlist}{\begin{list}{$\bullet$}{}}
\newcommand{\bboxlist}{\begin{list}{$\Box$}{}}
\newcommand{\bbboxlist}{\begin{list}{\raisebox{.005in}{{\tiny $\blacksquare$ \ \ }}}{}}
\newcommand{\bdashlist}{\begin{list}{$-$}{} }
\newcommand{\blist}{\begin{list}{}{} }
\newcommand{\barablist}{\begin{list}{\arabic{l1}}{\usecounter{l1}}}
\newcommand{\balphlist}{\begin{list}{(\alph{l2})}{\usecounter{l2}}}
\newcommand{\bAlphlist}{\begin{list}{\Alph{l2}.}{\usecounter{l2}}}
\newcommand{\bdiamlist}{\begin{list}{$\diamond$}{}}
\newcommand{\bromalist}{\begin{list}{(\roman{l3})}{\usecounter{l3}}}

\providecommand{\norm}[1]{\lVert#1\rVert}

\newcommand{\beq}{\begin{equation}}
\newcommand{\eeq}{\end{equation}}

\let\[\equation
\let\]\endequation

\newcommand{\tn}{\textnormal}

\DeclarePairedDelimiterX{\Norm}[1]{\lVert}{\rVert}{#1}

\DeclareMathOperator*{\argmin}{arg\,min}

\allowdisplaybreaks

\title{Online Learning for Predictive Control with Provable Regret Guarantees}

\author{Deepan Muthirayan, Jianjun Yuan, Dileep Kalathil, and Pramod P. Khargonekar
\thanks{This work is supported in part by the National Science Foundation under Grant ECCS-1839429.
D. Muthirayan and P. P. Khargonekar are with the Department of Electrical Engineering and Computer Sciences, University of California Irvine, Irvine, CA (emails: deepan.m@uci.edu, pramod.khargonekar@uci.edu). Jianjun Yuan is with the Expedia Group (email: yuanx270@umn.edu). Dileep Kalathil is with the Department of Electrical and Computer Engineering, Texas A\&M University (email:dileep.kalathil@tamu.edu).}
}

\begin{document}

\maketitle
\thispagestyle{empty}
\pagestyle{empty} 

\begin{abstract}
    We study the problem of online learning in predictive control of an unknown linear dynamical system with time varying cost functions which are unknown apriori. Specifically, we study the online learning problem where the control algorithm does not know the true system model and has only access to a fixed-length (that does not grow with the control horizon) preview of the future cost functions. The goal of the online algorithm is to minimize the dynamic regret, defined as the difference between the cumulative cost incurred by the algorithm and that of the best sequence of actions in hindsight. We propose two different online Model Predictive Control (MPC) algorithms to address this problem, namely Certainty Equivalence MPC (CE-MPC) algorithm and Optimistic MPC (O-MPC) algorithm. We show that under the standard stability assumption for the model estimate, the CE-MPC algorithm achieves $\mathcal{O}(T^{2/3})$ dynamic regret. We then extend this result to the setting where the stability assumption holds only for the true system model by proposing the O-MPC algorithm. We show that the O-MPC algorithm also achieves $\mathcal{O}(T^{2/3})$ dynamic regret,  at the cost of some additional computation. We also present numerical studies to demonstrate the performance of our algorithm.
\end{abstract}

\section{Introduction} 

The control of dynamical systems with uncertainties such as modeling errors, parametric uncertainty, and disturbances is a central challenge in control theory. There is vast literature in the field on control synthesis for systems with such uncertainties. The {robust control} literature studies the problem of feedback control  with modeling uncertainty and disturbances  \cite{ZhouDoyleGlover-RobustControl-Book} while the {adaptive control} literature studies the control of systems with parametric uncertainty  \cite{sastry2011adaptive}. Typically, these classical approaches are concerned with stability and asymptotic performance guarantees. 

Recently, there has been increasing attention on the online control algorithms for dynamical systems with uncertain disturbances, system parameters and cost functions. This online control literature focuses on the finite time performance guarantees of the algorithms \cite{dean2018regret, mania2019certainty, agarwal2019online, simchowitz2020improper}. The typical objective in these works is to minimize the \textit{static regret}, which is defined as the difference between the cumulative cost incurred by the online algorithm and the best policy from a certain class of policies. This is a key difference and challenge compared to the conventional adaptive control literature, and it requires combining techniques from statistical learning, online optimization and control. Most of the existing works in online control consider the setting where the online algorithm has access to only the past observations (of states, cost functions, and disturbances). On the other hand, in many practical problems such as robotics \cite{shi2019neural}, energy systems \cite{vazquez2016model}, data-center management \cite{lazic2018data} etc., a finite-length preview of the future cost functions and/or disturbances are available to the control algorithm to compute the current control input. The question then is, {\it how do we develop online control algorithms that can exploit this preview to provably achieve better performance guarantees}?

In the control theory literature, Model Predictive Control (MPC) addresses the class of problems where a preview of the future cost functions are available to compute the current control input. The MPC is a well studied methodology in the control literature \cite{rawlings2017model, borrelli2017predictive}. However, the MPC literature primarily focuses on asymptotic performance guarantees. In sharp contrast to these existing works, our goal is to develop an \textit{online learning MPC} algorithm with provable finite time performance guarantees. We focus on minimizing the metric of \textit{dynamic regret}, which is defined as the difference between the cumulative cost of the online algorithm and that of the optimal sequence of control actions in hindsight (with full information). Thus, the dynamic regret is a stronger performance metric compared to the static regret. Our objective is to show that an {\it optimally designed online learning MPC algorithm can achieve sub-linear dynamic regret using the preview information under minimal standard  assumptions}.

\textbf{Related work:} Recently, some works have addressed characterization of online performance (dynamic regret) of MPC algorithms \cite{li2019online, yu2020power, lin2021perturbation}. In \cite{li2019online}, the authors characterize the effect of preview on the dynamic regret of any baseline policy. They present an algorithm that improves the dynamic regret of any baseline policy exponentially with the length of the preview. In \cite{yu2020power}, the authors present guarantees for dynamic regret for a fixed LQR cost function with preview of disturbances. In \cite{lin2021perturbation}, the authors extend these results to strongly convex cost functions. However, these works: (i)  require the preview to grow at least logarithmically with the time horizon, and (ii) assume the system model is known. Significantly different from these works, we consider the more challenging setting of online learning in predictive control where {\it only a fixed-length preview is available} and \textit{the system model is unknown}.

Another recent work \cite{lale2021model} has also studied the problem of online learning in predictive control with a certainty equivalent approach, where the estimated parameter is directly used to compute the control policy. They consider a setting similar to ours, an unknown dynamical system without disturbances, with noise perturbed observations of the state. They show that their approach achieves $\mathcal{O}(T^{2/3})$ regret. However, this approach crucially depends on a major assumption that the MPC policy computed using the  estimated parameter is stabilizing for all systems within a certain radius from this estimate. In contrast, we develop a novel approach using the principle of optimism for  online learning predictive control, where the stability assumption is required only for the true underlying system, not for all the systems whose parameters lie within a ball around the true system parameters. We show that under this minimal and standard assumption, our proposed approach achieves  $\mathcal{O}(T^{2/3})$ dynamic regret.

Recently, many works have studied the online regret performance in control problems with  time-varying costs, disturbances and known system model \cite{abbasi2014tracking,cohen2018online,agarwal2019online,agarwal2019logarithmic,goel2019online}. A few others have also studied the problem with unknown linear systems. In \cite{dean2018regret}, the authors provide an algorithm for the LQR problem with unknown dynamics that achieves a regret of $\mathcal{O}(T^{2/3})$. In \cite{cohen2019learning}, the authors improve this result by providing an algorithm that achieves a regret of $\mathcal{O}(T^{1/2})$ for the same problem. In \cite{simchowitz2020improper}, the authors generalize these results to provide sub-linear regret guarantee for online control with partial observation for both known and unknown systems. However,  these works do not address online learning in predictive control which is the focus of this paper. 

\textbf{MPC:} Many MPC-based methods have been proposed for managing disturbances and uncertainties in the system dynamics. For example, some works handle disturbances or uncertainties by robust or chance constraints \cite{langson2004robust,goulart2006optimization,limon2010robust, tempo2012randomized, goulart2016robust}. Adaptive MPC techniques that adapt online when the system model is unknown have also been proposed \cite{fukushima2007adaptive, adetola2009adaptive, aswani2013provably,tanaskovic2019adaptive, bujarbaruah2019adaptive}. These methods primarily focus on constraint satisfaction, stability and in some cases performance improvements. In contrast to these works, {we consider non-asymptotic performance of an online MPC algorithm}. There are a  number  of works that provide performance analysis of MPC under both time-invariant costs \cite{angeli2011average, grune2014asymptotic, grune2015non} and time varying costs \cite{ferramosca2010economic, angeli2016theoretical, ferramosca2014economic, grune2017closed}. However, most of these studies also focus on asymptotic performance. 

\textbf{Main Contributions:} We address the problem of online learning in predictive control of an unknown linear dynamical system with time varying cost functions. We assume that \textit{the system model is unknown} to the control algorithm a priori and the control algorithm has only access to a \textit{fixed-length  preview} of the  future  cost functions. We propose a novel online learning MPC called Optimistic MPC (O-MPC) algorithm for this setting. We show that the O-MPC algorithm achieves a \textit{sublinear dynamic regret of $\mathcal{O}(T^{2/3})$, under a standard assumption used for establishing the asymptotic stability of MPC controllers for the true underlying system}. We also propose a computationally efficient algorithm called CE-MPC and prove that this algorithm achieves $\mathcal{O}(T^{2/3})$ under the extension of the standard stability assumption to the estimated model. \textit{To the best of our knowledge, this is the first work that gives a sub-linear dynamic regret guarantee for the online learning MPC problem with unknown system parameter and time varying cost functions under standard assumptions}.

\subsection{Notation}

We denote the spectral radius of a matrix $A$ by $\rho(A)$, the 2-norm of a vector by $\norm{\cdot}$, the Frobenious norm of a matrix $X$ by $\norm{X}_{F}$,  the non-negative part of the real line by $\mathbb{R}_{+}$, the discrete time interval from $m_1$ to $m_2$  by $[m_{1}, m_{2}]$, the sequence $(x_{m_1}, x_{m_1+1}, ..., x_{m_2})$ compactly by $x_{m_{1}:m_{2}}$. 
We denote the $j$th element of a vector $x$ by $x[j]$. We denote the $m-$ary cartesian power of a set $\mathcal{X}$ by $\mathcal{X}^m$. When a sequence of $m-$ dimensional vectors $x_{m_{1}:m_{2}}$ are i.i.d. over the support $\mathcal{X}^m$, we denote this by $x_t \stackrel{\tn{i.i.d}}{\sim} \mathcal{X}^m$, where it is implicit that $t \in [m_1,m_2]$.

\section{Problem Formulation and Preliminaries} 
\label{sec:problemformulation}

\subsection{Problem Statement}

We consider the online control of an {unknown} and {partially observed} linear dynamical system. The system evolution and the observation models are given by the equations
\begin{equation}
\label{eq:stateequation}
x_{t+1} = A^{\star} x_t +B^{\star} u_t, \hspace{0.5cm} ~~ y_t = x_t + \epsilon_t,
\end{equation}
where $x_{t} \in \mathbb{R}^n, u_t \in \mathbb{R}^m, y_{t} \in \mathbb{R}^n,\epsilon_t \in \mathbb{R}^n$ are the state of the system, control action, observation, and observation noise at time $t$, respectively. The system model is characterized by the parameters $ A^{\star} \in \mathbb{R}^{n \times n} $ and $ B^{\star} \in \mathbb{R}^{n \times m}$. For conciseness, we denote $\theta^{\star} = [A^{\star}, B^{\star}]$, and we assume that $\theta^{\star} \in \Theta \subset \mathbb{R}^{n \times (n+m)}$, where $\Theta$ is a known compact set.  

A control policy $\pi$ selects a control action $u^{\pi}_{t}$ at each time $t$ depending  on the available information, resulting in a sequence of  actions  $u^{\pi}_{1:T}$ and the corresponding state trajectory $x^{\pi}_{1:T}$. The cumulative cost of a policy $\pi$ under the system dynamics \eqref{eq:stateequation}  is given by 
\begin{align}
\label{eq:cumulative-cost}
 {J}_{T}(\pi; \theta^{\star}) = \sum_{t = 1}^T c_t(x^{\pi}_{t},u^{\pi}_t),
\end{align}
where $c_{t}$ is the cost function at time $t$. The typical goal is to  find the optimal policy $\pi^{\star}$ such that $\pi^{\star} = \argmin_{\pi} J_{T}(\pi; \theta^{\star})$. Computing $\pi^{\star}$ hence requires the knowledge of the system model $\theta^{\star}$ and the entire sequence of cost functions $c_{1:T}$. 

In most real-world control problems, it is not possible to find the optimal policy directly as described above because of two important  practical concerns: $(i)$ the true system  parameter $\theta^{\star}$ may be unknown to the decision maker a priori, and $(ii)$ the current and future cost functions, $c_{t:T}$, may be unknown to the decision maker at any time step $t$.  The  policy  that can achieve the minimum possible cumulative cost should then depend on the information available to the policy at  each time step.

In this work, we consider a setting where the decision maker (control policy) \textit{does not know the system parameter $\theta^{\star}$ a priori}, and has to learn the system parameter from the online observations.  Moreover, the policy has only access to a  \textit{fixed-length   preview of the next $M$ cost functions}, $c_{t:t+M-1}$, for making the control decision $u_{t}$ at each time step $t$. More precisely,  the  policy has only the following information available at each time $t$ for selecting the  action $u_{t}$: $(i)$ past observations $y_{1:t-1}$, current observation  $y_{t}$, and  past control inputs $u_{1:t-1}$, $(ii)$ past cost functions $c_{1:t-1}$,   and  $(iii)$ a {preview} of the next $M$ cost  functions $c_{t:t+M-1}$. The policy has to learn the unknown system parameter from the online observations and adapt with respect to the revealed future cost functions. So, such a policy is called an online learning policy. 

The performance of an online learning  policy $\pi$ is measured in terms of the \textit{dynamic regret}, defined as
\begin{align}
\label{eq:dynamic-regret-defn}
R_{T}(\pi) = J_{T}(\pi; \theta^{\star})  - J_{T}(\pi^{\star}; \theta^{\star}).
\end{align}
In other words, {dynamic regret} is the difference between the policy $\pi$ and that of the best policy $\pi^{\star}$ which has the complete information of the model parameter and loss functions. Our goal is to find an online learning policy that minimizes the dynamic regret. We note that the dynamic regret is a stronger performance metric compared to the more commonly used static regret  \cite{agarwal2019online, simchowitz2020improper}  where the cost of the online algorithm is compared with that of the best fixed policy from a specific class. 

We make the following assumptions on the system model.  
\begin{assumption}[System model]
\label{ass:system-model}
$(i)$ The set of possible system parameters $\Theta$ is a known compact set. Moreover, $\norm{\theta}_F \leq S, \ \forall \ \theta \in \Theta$.  \\
$(ii)$  $\rho(A^{\star}) < 1$, where $\rho(\cdot)$ is the spectral radius. The pair $(A^\star, B^\star)$ is controllable. \\
$(iii)$ 
The observation noise $\epsilon_{t}$  is uniformly bounded, i.e.,  $\norm{\epsilon_t}_2 \leq \epsilon_c,$~ $\forall t$.\\
$(iv)$ The cost functions $c_{1:T}$ are continuous and locally Lipschitz with a uniform Lipschitz constant for all $t$. 
\end{assumption}
The assumptions on the  boundedness of $\Theta$ and the spectral radius are standard in the online learning and control literature \cite{abbasi2011regret, dean2018regret, cohen2019learning, simchowitz2020improper}. 
Our assumption that the noise is bounded is similar to \cite{simchowitz2020improper}, which is the only other work that also studies online control of unknown systems with general cost functions. Similar to \cite{lale2021model}, which is the closest to our work, we also do not consider stochastic disturbances in our dynamics. The problem with disturbances is more challenging, although we believe the proof techniques we employ can be extended to the setting with the disturbances. We plan to address the problem with disturbance as a subsequent work. 


\subsection{Model Predictive Control: Preliminaries}

Model Predictive control is one of the most popular approaches for  control design when only a preview of the cost functions are available \cite{rawlings2017model, borrelli2017predictive}. The standard MPC algorithm uses the knowledge of the system parameter $\theta = [A(\theta), B(\theta)]$ to compute the optimal control sequence for that system. Algorithm \ref{alg:rhc} gives the formal description of the MPC algorithm with an $M$-step preview.  Given the current time step $t$, current state $x_{t}$, preview of the cost functions $c_{t:t+M-1}$, and the system parameter $\theta$ as the input, the MPC algorithm gives the control action $u_{t} = \texttt{MPC}(t, x_{t}, c_{t:t+M-1}, \theta)$ as the output. 
\begin{algorithm}[b!]
\begin{algorithmic}[1]
\caption{$\texttt{MPC}(t, x_{t}, c_{t:t+M-1}, \theta)$}
\label{alg:rhc}
\STATE With the initialization $\tilde{x}_{t} = x_{t}$, compute
\vspace{-0.2cm}
\begin{align*}
\bar{u}_{t:t+M-1} = \argmin_{\tilde{u}_{t:t+M-1}}~ &\sum_{k=t}^{t+M-1} c_k(\tilde{x}_k,\tilde{u}_k), \\
\text{s.t.}~&\tilde{x}_{k+1} = {A}(\theta)\tilde{x}_k+{B}(\theta) \tilde{u}_k
\end{align*}
\vspace{-0.6cm}
\STATE \textbf{Output}:  $\bar{u}_{t} $
\end{algorithmic}
\end{algorithm}
Unlike the optimal policy for the standard LQR problems, the optimal MPC policy for a linear system with general cost functions need not be  linear. So,  characterizing the stability properties of the MPC algorithm with general cost functions is much more challenging compared to the LQR setting. There has been significant works on analyzing the stability of systems that employ MPC policies under various assumptions \cite{lazar2008input, grimm2005model}. Clearly, any online learning MPC algorithm also has to ensure the stability of the system. So, we follow the same assumptions used in the literature that ensure stability of systems that employ MPC policies.

Define the $M$-step \textit{cost-to-go} function, denoted $V_{t}$, as
\begin{align}
\label{eq:cost-to-go}
\begin{split}
V_t(x; \theta) = &\min_{u_{t:t+M-1}} \sum^{t+M-1}_{k=t} c_{k}(x_k,u_k), \\
\text{s.t.}~ &x_{k+1} = A(\theta) x_k+B(\theta) u_k,~ x_{t} = x.
\end{split}
\end{align}

\begin{assumption}[Stability assumptions]
\label{as:stability}
There exist positive scalars $\underline{\alpha}, \overline{\alpha}$ and a continuous function $\sigma: \mathbb{R}^n \rightarrow \mathbb{R}_{+}$ such that: $(i)$ $c_t(x,u) \geq \underline{\alpha} \sigma(x),~ \forall x, \forall u, \forall t$,~ $(ii)$ $V_t(x; \theta) \leq \overline{\alpha}\sigma(x),~ \forall x, \forall t$,~ and  $(iii)$ $\lim_{\norm{x} \rightarrow \infty} \sigma(x) = \infty$.
\end{assumption}
Under the above assumption, \cite{grimm2005model} showed that the system with the parameter $\theta$ under the MPC policy has global asymptotic stability. We make use of the assumption in analyzing our approach. In contrast to other MPC approaches like \cite{lazar2008input}, the stability assumption we use from \cite{grimm2005model}, does not assume the existence of a Lyapunov-like function directly, which is a stronger stability assumption. 

We note that most of the works analyzing the stability of MPC policies  provide only asymptotic guarantees, including \cite{grimm2005model}. Our focus is on analyzing the finite-time performance of the MPC policy using the metric of dynamic regret in a more challenging setting with unknown system parameter.

We note that the prior online predictive control works such as \cite{li2019online, yu2020power, lin2021perturbation} do not use this stability assumption. However, they require a preview of length that is at least the logarithm of the whole control horizon ($O(\log T)$) to achieve sub-linear dynamic regret. With fixed-length preview, the algorithms proposed in \cite{li2019online, yu2020power, lin2021perturbation} will yield linear regret. This clearly shows the hardness of the fixed-length  preview setting. Unlike \cite{lale2021model}, which is the only other work that discusses online learning predictive control for fixed length preview like us, we do not require the stability assumption to hold for all the systems within a neighborhood around the true underlying system. Specifically, we show that sub-linear dynamic regret is achievable with the fixed-length preview under a stability assumption that is standard in MPC literature. 

\begin{remark}
The cost functions satisfying Assumption \ref{as:stability} include time varying quadratic function of the type $c_t(x) = (x-b)^\top Q_t (x-b) + u^\top R_tu$, where $b$ is an offset, and non-convex functions of the state (see \cite{grimm2005model} for examples). Thus, the cost functions we consider are quite general.
\end{remark}

\section{Algorithms and Regret Performance Guarantees} 
\label{sec:algorithm}

In this section, we present two different online learning MPC algorithms, namely Certainty Equivalence MPC (CE-MPC) algorithm and Optimistic MPC (O-MPC) algorithm. Both algorithms operate in two phases: $(i)$ exploration phase, and $(ii)$ control phase. In the exploration phase, both algorithms follow a pure exploration strategy to estimate the unknown system parameter. In the control phase, the algorithms employ an MPC policy with the parameter estimated using the observation from the exploration phase. The algorithms, however, differ in the parameter estimation approach. The CE-MPC algorithm uses a certainty equivalence approach  which treats the parameter estimate as the true parameter and employs an MPC policy based on this parameter. The O-MPC algorithm selects an optimistic parameter from a high confidence region around the parameter estimate and employs an MPC policy based on this optimistic parameter. 

Both the CE-MPC algorithm and the O-MPC algorithm provide the same regret guarantees, but under two different assumptions. The CE-MPC algorithm is computationally more tractable than the O-MPC algorithm, but requires a stronger assumption for the regret guarantee. The O-MPC algorithm eliminates the need for this assumption, but at the expense of additional computational complexity. The algorithms are presented in Algorithm \ref{alg:CE-RHC} and  Algorithm \ref{alg:O-RHC}. 

\subsection{Exploration Phase}

The goal of the exploration phase is to explore the system and collect the observations to estimate the unknown parameter $\theta^{\star}$ upto a desired accuracy with high probability. This is achieved by a pure exploration strategy for the first $T_{0}$ steps.

We adopt the approach of \cite{hazan2020nonstochastic}. Both the algorithms set the control actions in the exploration phase as an i.i.d. random sequence as given by 
\beq 
u_t \stackrel{\tn{i.i.d.}}{\sim} \{ \pm 1\}^m,
\label{eq:control-sequence-estimation}
\eeq 
with $\mathbb{P}(u_t[j] = \pm 1) = 1/2$, where $\{ \pm 1\}$ denotes the set with elements $1$ and $-1$ and $\{ \pm 1\}^m$ denotes the $m-$ary cartesian power. 

At the end of the exploration phase, both the algorithms compute the following quantities
\beq 
N_j = \frac{1}{T_0-n}\sum_{t = 0}^{T_0-n-1} y_{t+j+1}u^\top_t, ~ \forall ~ j \in [1,n]. \nonumber 
\eeq 
Let
\beq 
\widetilde{C}_0 = [N_0~ N_1 \dots N_{n-1}], ~ \widetilde{C}_1 = [N_1~N_2 \dots N_{n}]. \nonumber
\eeq 

Then, both the algorithms compute the estimate of $\theta^\star$ as 
\beq 
\label{eq:lse-estimation}
\hat{\theta}_{\mathrm{ls}} = [\hat{A} ~\hat{B}] ~ \hat{B} = N_0, ~ \hat{A} = \widetilde{C}_1\widetilde{C}^\top_0(\widetilde{C}_0\widetilde{C}^\top_0)^{-1}.
\eeq 

Given the parameter estiamete $\hat{\theta}_{\mathrm{ls}}$, the algorithms compute a high confidence set as given by
\beq
\label{eq:confidence-region}
\widehat{\Theta} = \left\{\theta: ~ \|\theta - \hat{\theta}_{\mathrm{ls}} \|_F \leq \beta({\delta})\right\},
\eeq
where 
\begin{align}
\label{eq:beta-confidence}
\beta(\delta) = \sqrt{\frac{2\times 10^3n^2\kappa^8(\sqrt{m}\epsilon_c+c_\rho\gamma_\rho^{-1}mS)^2\log(mn^2/\delta)}{T_0}}.
\end{align}
Here, the constants $c_\rho$ and $\gamma_\rho$ are constants such that $\norm{{A^\star}^k} \leq c_\rho (1-\gamma_\rho)^k$, and $\kappa$ is a sufficiently large constant such that $ \norm{(C_0C^\top_0)^{-1}} \leq \kappa$, where $C_0 = [B, \dots, A^{n-1}B]$. The implicit assumption here is that these constants are known and can be used to construct the high confidence set. We note that this assumption is essential to construct the high confidence set.

\begin{proposition}
\label{prop:estimation}
Suppose Assumption \ref{ass:system-model} holds, and $u_{1:T_{0}}$ is given by \eqref{eq:control-sequence-estimation}. Then, for sufficiently large $T_0$, with probability greater than $1- \delta, ~ \theta^\star \in \widehat{\Theta}$.
\end{proposition}
We note that the lower bound for $T_{0}$ given in the above proposition does not depend on the horizon $T$. The proof is exactly the same as the proof of \cite[Theorem 19]{hazan2020nonstochastic}. We refer the reader to \cite[Theorem 19]{hazan2020nonstochastic} for the details of the proof.

\subsection{Control Phase} 

In the control phase, both algorithms employs an MPC policy, but with different estimates for the true model. 

\subsubsection{CE-MPC Algorithm}
The CE-MPC  algorithm treats the parameter estimate $\hat{\theta}_{\mathrm{ls}}$  as the true parameter and selects control actions according to the standard  MPC algorithm. More precisely, at each time $t \in [T_{0}+1, T]$, the CE-MPC algorithm takes the control action
\begin{align}
\label{eq:ce-control-action}
{u}^{\pi_{\mathrm{ce}}}_t = \texttt{MPC}(t, \hat{x}_{t},  c_{t:t+M-1}, \hat{\theta}_{\mathrm{ls}}),
\end{align}
where $\hat{x}_{t} = A(\hat{\theta}_{\mathrm{ls}})\hat{x}_{t-1} + B(\hat{\theta}_{\mathrm{ls}}) {u}^{\pi_{\mathrm{ce}}}_{t-1}$, with the initialization $\hat{x}_{T_{0}+1} = y_{T_{0}+1}$. The CE-MPC algorithm is formally presented in Algorithm \ref{alg:CE-RHC}. 
\begin{algorithm}[]
\begin{algorithmic}[1]
\caption{CE-MPC Algorithm}
\label{alg:CE-RHC}
\STATE \textbf{Input}: $T_{0}, \delta$
\STATE \textbf{Estimation Phase:}
\FOR {$t = 1, \ldots, T_{0}$}
\STATE Select the control action $u_{t}$ according to  \eqref{eq:control-sequence-estimation}
\ENDFOR
\STATE Estimate the parameter $\hat{\theta}_{\mathrm{ls}}$ according to  \eqref{eq:lse-estimation}  
\STATE Set $\hat{x}_{T_0+1} = y_{T_0+1}$
\STATE \textbf{Control Phase:}
\FOR {$t = T_{0}+1, \ldots, T$}
\STATE Select the control action \\${u}^{\pi_{\mathrm{ce}}}_t = \texttt{MPC}(t, \hat{x}_{t},  c_{t:t+M-1},\hat{\theta}_{\mathrm{ls}})$
\STATE $\hat{x}_{t+1} =  A(\hat{\theta}_{\mathrm{ls}})\hat{x}_t +  B(\hat{\theta}_{\mathrm{ls}}) {u}^{\pi_{\mathrm{ce}}}_t$
\ENDFOR
\end{algorithmic}
\end{algorithm}

\subsubsection{O-MPC Algorithm}

The O-MPC  algorithm uses an optimistic approach that simultaneously selects the optimistic parameter from the confidence region $\widehat{\Theta}$ and the optimal control action with respect to this optimistic parameter. More precisely, at each time $t \in [T_{0}+1:T]$, the O-MPC algorithm selects the control action ${u}^{\pi_{\mathrm{o}}}_t$ optimistic parameter $\hat{\theta}_{t}$ as
\begin{align}
\label{eq:o-control-action}
({u}^{\pi_{\mathrm{o}}}_t, \hat{\theta}_{t}) = \texttt{O-MPC}(t, \hat{x}_{t},  c_{t:t+M-1}, \widehat{\Theta}),
\end{align}
where the \texttt{O-MPC} subroutine is given in Algorithm \ref{alg:o-rhc-subroutine}. The complete O-MPC algorithm is formally presented in Algorithm \ref{alg:O-RHC}. 

\begin{algorithm}[]
\begin{algorithmic}[1]
\caption{$\texttt{O-MPC}(t, \hat{x}_{t}, c_{t:t+M-1}, \widehat{\Theta})$}
\label{alg:o-rhc-subroutine}
\STATE With the initialization $\tilde{x}_{t} = x_{t}$, compute
\vspace{-0.2cm}
\begin{align*}
(\bar{u}_{t:t+M-1}, \theta_{t})  = \argmin_{\tilde{u}_{t:t+M-1}, \theta \in \widehat{\Theta}}~ &\sum_{k=t}^{t+M-1} c_k(\tilde{x}_k,\tilde{u}_k), \\
\text{s.t.}~\tilde{x}_{k+1} &= {A}(\theta)\tilde{x}_k+{B}(\theta) \tilde{u}_k
\end{align*} 
\vspace{-0.4cm}
\STATE \textbf{Output}:  $(\bar{u}_{t}, \theta_{t})$
\end{algorithmic}
\end{algorithm}

\begin{algorithm}[]
\begin{algorithmic}[1]
\caption{O-MPC Algorithm}
\label{alg:O-RHC}
\STATE \textbf{Input}: $T_{0}, \delta$
\STATE \textbf{Estimation Phase:}
\FOR {$t = 1, \ldots, T_{0}$}
\STATE Select control action $u_{t}$ according to  \eqref{eq:control-sequence-estimation}
\ENDFOR
\STATE Estimate the parameter $\hat{\theta}_{\mathrm{ls}}$ according to  \eqref{eq:lse-estimation}  
\STATE Estimate the confidence region $\widehat{\Theta}$ according to  \eqref{eq:confidence-region} -\eqref{eq:beta-confidence} 
\STATE Set $\hat{x}_{T_0+1} = y_{T_0+1}$
\STATE \textbf{Control Phase:}
\FOR {$t = T_{0}+1, \ldots, T$}
\STATE Select the control action ${u}^{\pi_{\mathrm{o}}}_t$ and the optimistic parameter estimate $\hat{\theta}_{t}$ as \\$({u}^{\pi_{\mathrm{o}}}_t, \hat{\theta}_{t})= \texttt{O-MPC}(t, \hat{x}_{t},  c_{t:t+M-1},\widehat{\Theta})$
\STATE $\hat{x}_{t+1} = {A}(\hat{\theta}_{t}) \hat{x}_t + {B}(\hat{\theta}_{t}) {u}^{\pi_{\mathrm{o}}}_t$
\ENDFOR
\end{algorithmic}
\end{algorithm}

\subsection{Regret Performance Guarantees}

We now formally present the regret guarantees of the CE-MPC algorithm (denoted as $\pi_{\mathrm{ce}}$) and O-MPC algorithm  (denoted as $\pi_{\mathrm{o}}$)

\begin{theorem}[Regret of the CE-MPC Algorithm]
\label{thm:ce-rhc-thm}
Suppose Assumption  \ref{ass:system-model} holds and $\hat{\theta}_{\mathrm{ls}}$ satisfies the conditions given in Assumption \ref{as:stability}. Suppose $M > (\overline{\alpha}/\underline{\alpha})^2 + 1$, $T_{0} = T^{2/3}$ and $T_0$ satisfies the conditions in Proposition \ref{prop:estimation}. Then, with probability greater than $(1-\delta)$, 
\begin{align*}
R_{T}(\pi_{\mathrm{ce}}) \leq \mathcal{O}(T^{2/3})
\end{align*}
\end{theorem}

\begin{theorem}[Regret of the O-MPC Algorithm]
\label{thm:o-rhc-thm}
Suppose Assumption \ref{ass:system-model} holds, $\theta^{\star}$ satisfies Assumption \ref{as:stability},  and $T_0$ satisfies the conditions in Proposition \ref{prop:estimation}. Assume that $\overline{\alpha}/\underline{\alpha} < 2$. Fix $T_{0} = T^{2/3}$. Then there exists $M, \overline{T}$ such that, for $T > \overline{T}$, with probability greater than $(1-\delta)$
\begin{align*}
R_{T}(\pi_{\mathrm{o}}) \leq \mathcal{O}(T^{2/3})
\end{align*}
\end{theorem}

\begin{remark}
We observe that the dynamic regret guarantee is valid only when the preview $M$ is greater than the threshold given by $(\overline{\alpha}/\underline{\alpha})^2+1$. Such a lower bound requirement is typical in MPC algorithms; see for example \cite{grimm2005model}. This is expected because, the control computed from a short preview might be very inaccurate and can potentially lead to instability.
\end{remark}

\begin{remark}
The first theorem implies that CE-MPC achieves sub-linear dynamic regret provided the stability assumption is satisfied by the parameter estimate $\hat{\theta}_{\mathrm{ls}}$. This is feasible provided a neighborhood of parameters around $\theta^\star$ satisfy the assumption. In contrast, the O-MPC algorithm achieves the same dynamic regret under a significantly milder condition that the stability assumption holds only for the system parameter $\theta^\star$. We note that O-MPC achieves this at the cost of the additional computation to estimate the optimistic model alongside the control input. We show later in the simulations that the CE-MPC is as effective as O-MPC in many practical problems.
\end{remark}

\begin{remark}
The condition on the ratio of upper and lower bound to the cost functions, i.e., $\overline{\alpha}/\underline{\alpha} < 2$ can be relaxed by making additional but less restrictive assumptions. The ratio condition is only required to establish the boundedness of the state, as we shall show in the proof later. 
\end{remark}

\section{Regret Analysis} 
\label{sec:analysis}

Given a sequence of system parameters $\theta_{T_{1}:T_{2}}$, a sequence of control actions $u_{1:T}$, and an initial state $x$,  we define $J_{T_{1}:T_{2}}(u_{T_{1}:T_{2}}; \theta_{T_{1}:T_{2}})$ as  
\begin{align}
J_{T_{1}:T_{2}}(u_{T_{1}:T_{2}}; \theta_{T_{1}:T_{2}}) &= \sum^{T_{2}}_{t=T_{1}} c_{t}(x_{t}, u_{t}), \nonumber \\
\text{s.t.}~~ x_{t+1} = A(\theta_{t}) x_{t} &+ B(\theta_{t}) u_{t},~~x_{T_{1}} = x, 
\end{align}
for any $T_{1}, T_{2} \in [1, T]$. We make the dependence on the initial state $x$ implicit as it will be clear from the context. If $\theta_{t} = \theta, \forall t \in [T_{1}, T_{2}]$, we will simplify the above notation as $J_{T_{1}:T_{2}}(u_{T_{1}:T_{2}}; \theta)$.

Let $u^{\pi}_{1:T}$ and $u^{\pi^{\star}}_{1:T}$ be the sequence of control actions generated by the policies $\pi$ and $\pi^{\star}$, respectively.  For analyzing the regret, we decompose it  into three terms as follows:
\begin{align}
&R_{T}(\pi)  =  \underbrace{ J_{1:T_{0}}(u^{\pi}_{1:T_{0}}; \theta^{\star}) -  J_{1:T_{0}}(u^{\pi^{\star}}_{1:T_{0}};\theta^{\star})}_{\mathrm{Term~I}} \nonumber \\
&+ \underbrace{ J_{T_{0}+1:T}(u^{\pi}_{T_{0}+1:T}; \theta^{\star}) -  J_{T_{0}+1:T}(u^{\pi}_{T_{0}+1:T}; \hat{\theta}_{T_{0}+1:T})}_{\mathrm{Term~II}}   \nonumber \\
\label{eq:term-splitting}
&+ \underbrace{ J_{T_{0}+1:T}(u^{\pi}_{T_{0}+1:T}; \hat{\theta}_{T_{0}+1:T}) -  J_{T_{0}+1:T}(u^{\pi^{\star}}_{T_{0}+1:T};\theta^{\star})}_{\mathrm{Term~III}}  
\end{align}
We characterize the regret due to each term separately for both policies $\pi_{\mathrm{ce}}$ and $\pi_{\mathrm{o}}$. Note that for  $\pi_{\mathrm{ce}}$, $\hat{\theta}_{t} =\hat{\theta}_{\mathrm{ls}}, \forall t \in [T_{0}+1, T]$. 

\subsection{Regret of Term I}

Term I characterizes the regret due to the exploration phase. We show that,  under the exploration strategy we use, the regret due to exploration is bounded by the length of the exploration phase.  Since the exploration strategy is identical for both the CE-MPC algorithm and the O-MPC algorithm, regret of term I is also identical for both algorithms.

The key challenge involved here is to show that  the system state does not grow unbounded during the exploration phase. For this, we make use of the fact that  the spectral radius of $A^{\star}$ is strictly less than one and the control sequences are bounded. We then use the fact that the cost functions are locally Lipschitz to show that the realized cost at each time step of the estimation phase is bounded. From here, it follows that the regret of Term I is $\mathcal{O}(T_0)$. We formally state the result below. 
\begin{proposition}[Regret of Term I]
\label{propo:TermI-ce}
Suppose Assumption \ref{ass:system-model} holds. Let  $\mathrm{Term~I}$ be as defined in \eqref{eq:term-splitting}. Then, for $\pi \in \{\pi_{\mathrm{ce}}, \pi_{\mathrm{o}}\}$, 
\begin{align*}
\mathrm{Term~I} \leq \mathcal{O}(T_{0})
\end{align*}
\end{proposition} 
Note that, if we set $T_{0} = T^{2/3}$ as specified in Theorem \ref{thm:ce-rhc-thm}, then the regret due to Term I is $\mathcal{O}(T^{2/3})$. 

\subsection{Regret of Term II}

The CE-MPC algorithm generates the sequence of control actions $u^{\pi_{\mathrm{ce}}}_{T_{0}+1:T}$ using the parameter  estimate $\hat{\theta}_{\mathrm{ls}}$. However, these control actions are applied on the true system with parameter $\theta^{\star}$. Term II characterizes the regret due to this estimation error.

To analyze this term, we first show that, contingent on the states being bounded, the states of any two systems driven by the same sequence of control actions differ by a term that is bounded by the norm of the difference of the parameters of the two systems. Recall that, in Proposition \ref{prop:estimation}, we proved that $\|\theta^{\star} - \hat{\theta}_{\mathrm{ls}}\| \leq O(1/\sqrt{T_{0}})$ with high probability. 

We then separately show that states are indeed bounded under CE-MPC algorithm when Assumption \ref{as:stability} holds for $\hat{\theta}_{\mathrm{ls}}$. This also implies that the control actions are bounded. The cost functions being locally Lipschitz and the states and control actions being bounded, the cumulative cost can now be upperbounded by the length of the horizon $T-T_{0}$. Combining this with the observation made in the above paragraph, we get a net upperbound $O(T/\sqrt{T_{0}})$. We state this result formally below. 
\begin{proposition}[Regret of Term II for CE-MPC]
\label{propo:TermII-ce}
Suppose Assumption \ref{ass:system-model} holds and $\hat{\theta}_{\mathrm{ls}}$ satisfies Assumption \ref{as:stability}. Suppose $T_0$ satisfies the conditions in Proposition \ref{prop:estimation} and $M > (\overline{\alpha}/\underline{\alpha})^2 + 1$. Let  $\mathrm{Term~II}$ be as defined in \eqref{eq:term-splitting} and let $\pi = \pi_{\mathrm{ce}}$. Then, with probability greater than $(1-\delta)$,
\begin{align*}
\mathrm{Term~II} \leq \mathcal{O}({T}/{\sqrt{T_{0}}})
\end{align*}
\end{proposition}
Here also, if we set $T_{0} = T^{2/3}$, the regret of Term II is $\mathcal{O}(T^{2/3})$. 

The analysis of Term II is more challenging for the O-MPC algorithm because we need to consider the sequence of parameters $\hat{\theta}_{T_{0}+1:T}$ selected by the algorithm.  We overcome this issue by first characterizing an upperbound on  the difference of the states of the two systems at a time $t$  by a decaying sum of the parameter difference from $t$ to $T_0+1$. Since $\|\theta^{\star} - \hat{\theta}_{t}\| \leq O(1/\sqrt{T_{0}})$ for all $t$ using   Proposition \ref{prop:estimation}, we can now use techniques similar to the proof of Proposition \ref{propo:TermII-ce} to establish the following bound.
\begin{proposition}[Regret of Term II for O-MPC]
\label{propo:TermII-o}
Suppose Assumption \ref{ass:system-model} holds, $\theta^{\star}$ satisfies Assumption \ref{as:stability},  and $T_0$ satisfies the conditions in Proposition \ref{prop:estimation}. 
Assume that $\overline{\alpha}/\underline{\alpha} < 2$. Let  $\mathrm{Term~II}$ be as defined in \eqref{eq:term-splitting} and let $\pi = \pi_{\mathrm{o}}$.  Then there exists $M, \overline{T}$ such that, for $T > \overline{T}$, with probability greater than $(1-\delta)$
\begin{align*}
\mathrm{Term~II} \leq \mathcal{O}({T}/{\sqrt{T_{0}}})
\end{align*}
\end{proposition}
Here also, if we set $T_{0} = T^{2/3}$, the regret of Term II is $\mathcal{O}(T^{2/3})$. 

\subsection{Regret of Term III}
For the CE-MPC algorithm, we bound Term III by bounding its first term, which  is the cumulative cost for the standard  MPC controller for a system with parameter $\hat{\theta}_{\mathrm{ls}}$. To bound this term, we use the stability assumption for the estimated parameter $\hat{\theta}_{\mathrm{ls}}$. We note that the stability assumption  does not directly imply the existence of a Lyapunov-like function. The key part of the proof is in establishing that under Assumption \ref{as:stability}, for sufficiently large $M$, the function $V_t(\cdot;\hat{\theta}_{\mathrm{ls}})$ \eqref{eq:cost-to-go}  becomes a Lyapunov-like function. This guarantees exponential convergence for the system resulting in $\mathcal{O}(1)$ regret for Term III for CE-MPC. We formally state this result below. \begin{proposition}[Regret of Term III for CE-MPC]
\label{propo:TermIII-ce}
Suppose Assumption \ref{ass:system-model} holds and $\hat{\theta}_{\mathrm{ls}}$ satisfies Assumption \ref{as:stability}. Let $M > (\overline{\alpha}/\underline{\alpha})^2 + 1$. Let  $\mathrm{Term~III}$ be as defined in \eqref{eq:term-splitting} and let $\pi = \pi_{\mathrm{ce}}$. Then,
\begin{align*}
\mathrm{Term~III}   \leq \mathcal{O}(1)
\end{align*}
\end{proposition}

\begin{figure*}[t!]
    \centering
	\begin{minipage}{.31\textwidth}
		\centering
		\includegraphics[width=\linewidth]{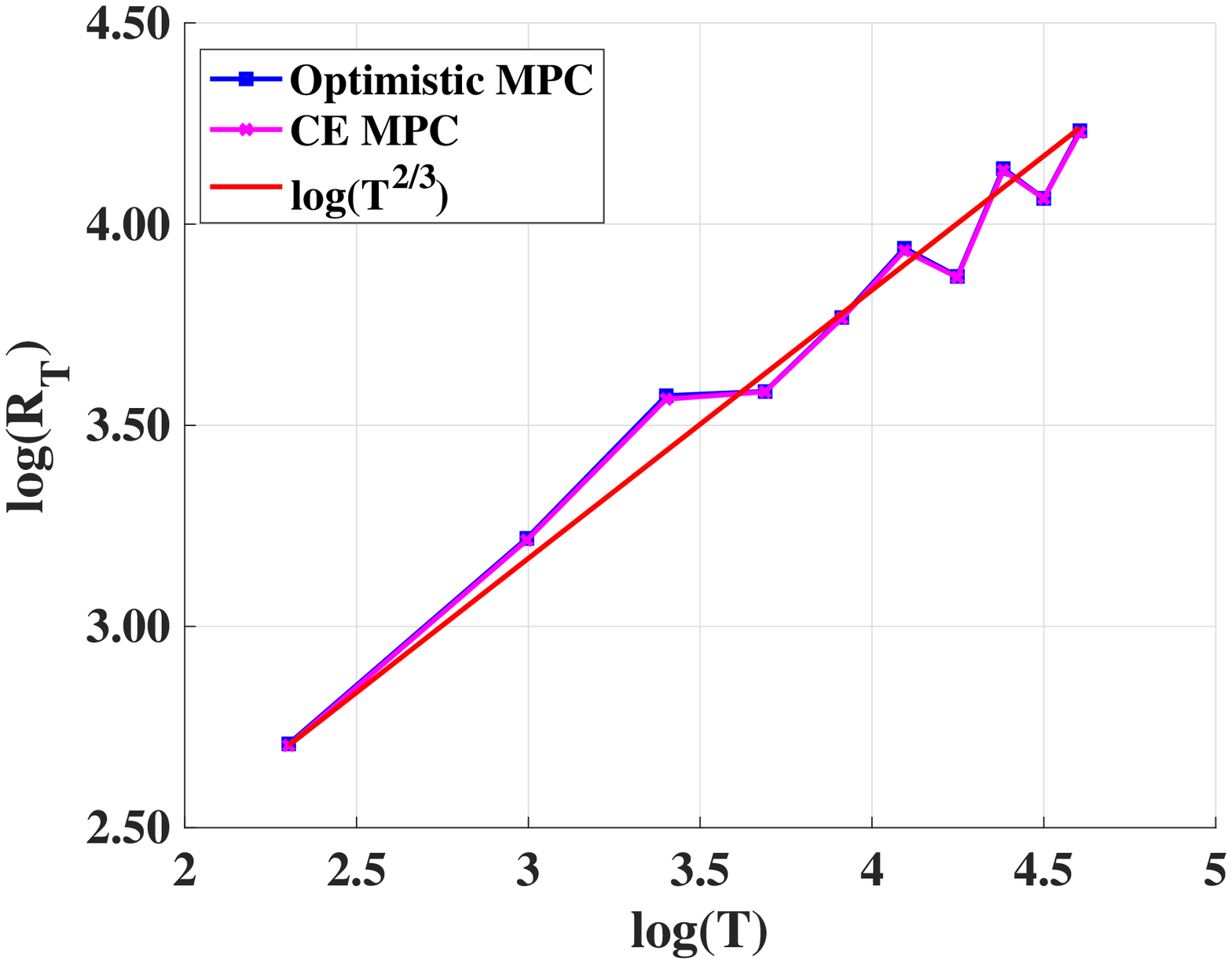}
	\end{minipage}
	\begin{minipage}{.31\textwidth}
		\centering
		\includegraphics[width=\linewidth]{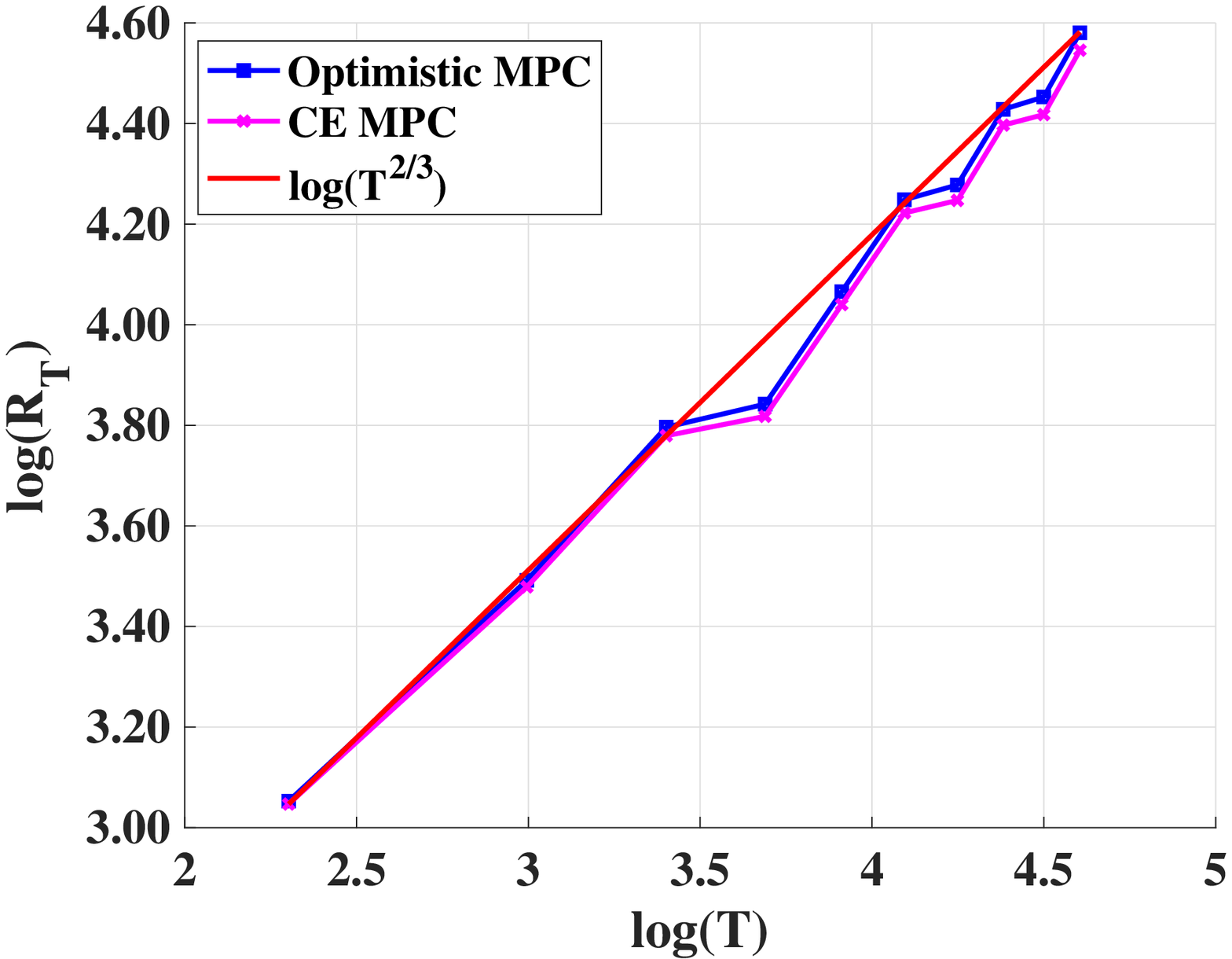}
	\end{minipage}
	\begin{minipage}{.31\textwidth}
		\centering
		\includegraphics[width=\linewidth]{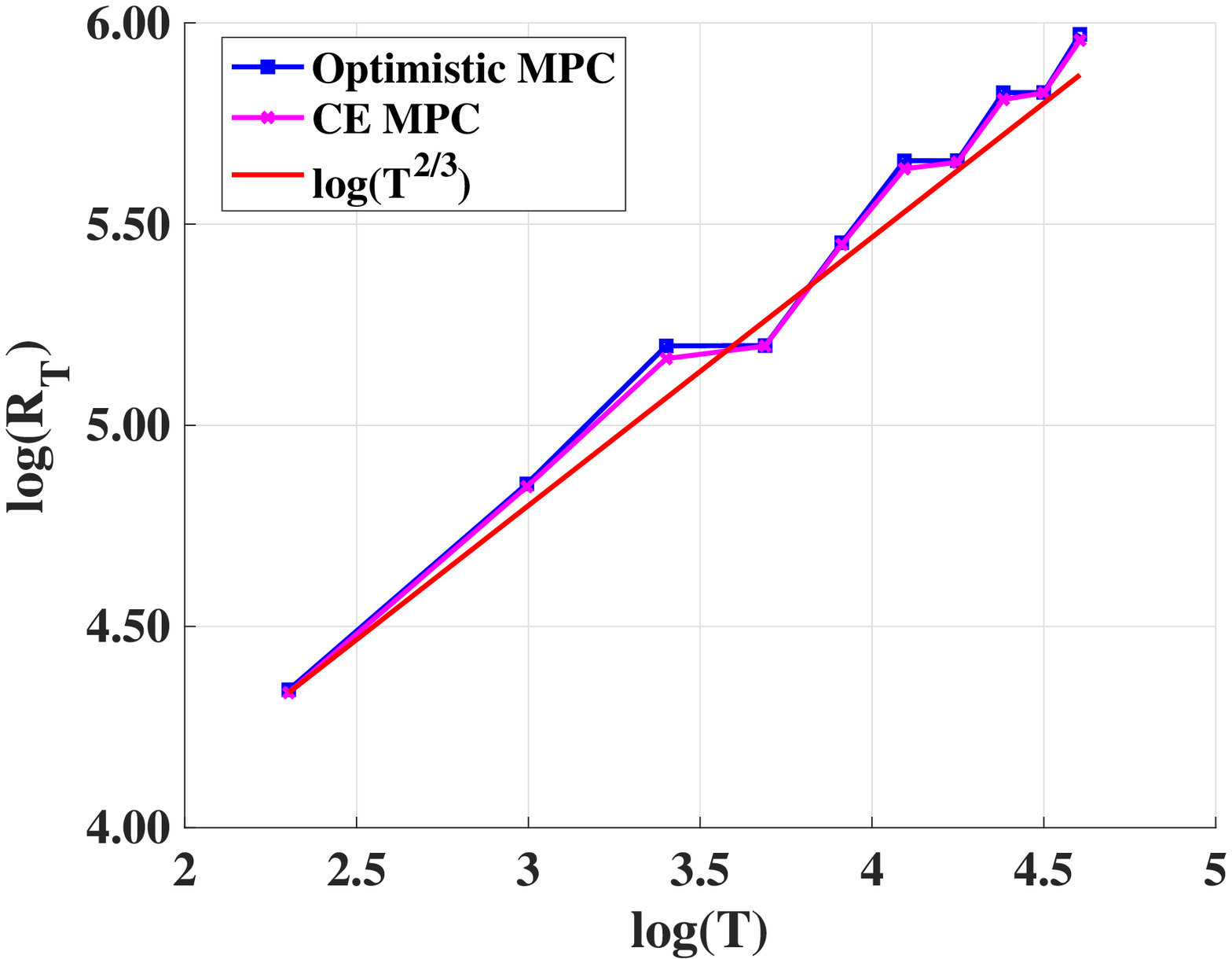}
	\end{minipage}
	\centering
\caption{Plot of variation of $\log{R_T}$ as a function of $\log{T}$. Left: $c_t(x,u) = (x-b)^\top Q_t(x-b) + u^\top R_t u$, where $b = [0.01,0.01]^\top$, and $Q_t$ and $R_t$ are random with their elements in the range $[0.375,0.625]$. Middle: $c_t(x,u) = \sigma_\mathcal{X}(x) + u^\top u$, where $\mathcal{X}$ is the ball of radius $0.25$ centered at $0.5$. Right: $c_t(x,u) = \vert (x(1)-b)\vert ^3 + (x(2)-b)^2 + u^\top u$, where $b = 0.1$.}
\label{fig:regret}
	
\end{figure*}

The proof for the O-MPC algorithm is significantly more challenging because we assume that  the stability assumption  is true only for the true system $\theta^\star$, whereas the quantity to be analyzed is the cumulative cost of the MPC controller for a time-varying system. The proof uses the fact that estimate is optimistic to leverage Assumption \ref{as:stability} satisfied by $\theta^\star$. This leads to a Lyapunov-like condition with an additional  term that is proportional to the difference between $\theta^\star$ and $\hat{\theta}_t$. The novelty of the proof technique is how the optimistic estimate is used to establish the Lyapunov-like condition. This additional  term at every time step leads to the $\mathcal{O}(T/\sqrt{T_0})$ overall regret instead of $\mathcal{O}(1)$ as in CE-MPC.
\begin{proposition}[Regret of Term III for O-MPC]
\label{propo:TermIII-o}
Suppose Assumption \ref{ass:system-model} holds, $\theta^{\star}$ satisfies Assumption \ref{as:stability},  and $T_0$ satisfies the conditions in Proposition \ref{prop:estimation}. 
Assume that $\overline{\alpha}/\underline{\alpha} < 2$. Let  $\mathrm{Term~III}$ be as defined in \eqref{eq:term-splitting} and let $\pi = \pi_{\mathrm{o}}$.  Then there exists $M, \overline{T}$ such that, for $T > \overline{T}$, with probability greater than $(1-\delta)$
\begin{align*}
\mathrm{Term~III} \leq \mathcal{O}({T}/{\sqrt{T_{0}}})
\end{align*}
\end{proposition}
Note that, by setting $T_{0} = T^{2/3}$, the regret of Term III becomes $\mathcal{O}(T^{2/3})$. 

\subsection{Proofs of the Main Results}
Proof of the main results now immediately follow by using the upperbounds obtained for Term I, II, and III. We state this formally below. 
\begin{proof}[Proof of Theorem \ref{thm:ce-rhc-thm}]
The proof follows by setting $T_0 = T^{2/3}$ and combining  Proposition \ref{propo:TermI-ce}, Proposition \ref{propo:TermII-ce} and Proposition \ref{propo:TermIII-ce}. 
\end{proof}

\begin{proof}[Proof of Theorem \ref{thm:o-rhc-thm}]
The proof follows by setting $T_0 = T^{2/3}$ and combining  Proposition \ref{propo:TermI-ce}, Proposition \ref{propo:TermII-o} and Proposition \ref{propo:TermIII-o}.
\end{proof}

\section{Numerical Experiments}

In this section we present three numerical examples to illustrate the performance of the O-MPC and CE-MPC algorithm. In all examples, we consider a linear dynamical system as  given in   \eqref{eq:stateequation} with $n = 2$ and $m = 1$. In each example, the system matrices $A^{\star}$ and $B^{\star}$ are chosen randomly, with the elements of $A^{\star}$ in the range $[0,0.5]$ and the elements of $B^{\star}$ in the range $[0,1]$. The preview $M$ is set to be $5$ in all the examples. This $M$ is computed as  $\left(\overline{\alpha}/\underline{\alpha}\right)^2+1$, as given in  Theorem \ref{thm:o-rhc-thm}, for the fixed quadratic cost given by $Q = I$ and $R = I$. We select different sets of cost functions for each example, as:

\textit{Example 1:} we select quadratic cost functions, $c_t(x,u) = (x-b)^\top Q_t(x-b) + u^\top R_t u$, where $b = [0.01, 0.01]^\top$, and $Q_t$ and $R_t$ are randomly chosen diagonal matrices with each diagonal element lying in the range $[0.375, 0.625]$. This example hence illustrates a specific case of online predictive linear quadratic control with time varying cost functions and unknown system model.  

\textit{Example 2:} we select the sequence of cost functions given by $c_t(x,u) = \sigma_{\mathcal{X}}(x) + u^\top u$, where $\sigma_{\mathcal{X}}(x) := \inf_{y \in \mathcal{X}} \| y -  x\|^{2}$, and $\mathcal{X}$ is the ball of radius $0.25$ centered at $0.5$. This example focuses on the convergence to a specific region in the state space characterized by $\mathcal{X}$, which is often an important objective in predictive control.

\textit{Example 3:} we select non-quadratic cost functions given by  $c_t(x,u) = \vert (x[1]-b)\vert ^3 + (x[2]-b)^2 + u^\top u$,  where $b = 0.1$. This example illustrates a specific case of online predictive control with non-convex cost functions. While most of the online predictive control techniques are reliant on convexity assumption, this assumption need not hold in all control problems (see \cite{grimm2005model} for details). 

We note that, in  these examples, because the costs are offset from zero, $u_t = 0$ will not achieve a sub-linear regret.

The variation of the regret for all these examples is shown in Fig. \ref{fig:regret}. We find that the scaling of the regret of O-MPC in all these examples matches with our theoretical guarantee. We also find that the regret of CE-MPC closely matches O-MPC in all these examples. This shows that we can achieve the same level of performance as O-MPC in many practically relevant examples by using the more computationally efficient CE-MPC algorithm in place of O-MPC. 
\section{Conclusion}

In this work, we present online learning and control algorithms for model predictive control of linear dynamical systems under standard system assumptions. Our work sheds light on methods, conditions and analysis for predictive control of unknown systems with limited preview. We show that by using a stability assumption that is standard in the asymptotic analysis of MPC algorithms,  we can guarantee $\mathcal{O}(T^{2/3})$ dynamic regret for this setting. In future, we plan to extend this algorithm and analysis to systems with adversarial disturbances.

\bibliographystyle{IEEEtran} 
\bibliography{Refs-OnlineMPC}

\begin{appendices}

\section{Proof of the Results in Section \ref{sec:analysis}}

\subsection{Proof of Proposition \ref{propo:TermI-ce}}
\label{sec:proof-TermI-ce}

\begin{proof}
In the estimation phase, the control input given by  \eqref{eq:control-sequence-estimation} is clearly bounded. Also by Assumption \ref{ass:system-model}.$(ii)$, the system is stable. This then implies that $x^{\pi}_t$ is bounded for all $t \in [1:T_0]$. Let
$b = \max_{t \in [1, T_{0}]} \|x^{\pi}_t\|$. We note that $b$ is a constant that does not increase with $T_{0}$. This follows from the fact that $\norm{u^{\pi}_t} \leq (n+1)m$ throughout the estimation phase and $\rho(A^\star)  < 1$. 
Let $L_{c}$ be the uniform Lipschitz constant for all $c_t$ over the closed and bounded set $\{(x,u): \norm{x} \leq b, \norm{u} \leq (n+1)m\}$. It then follows that $c_t(x^{\pi}_t,u^{\pi}_t) \leq L_{c} (\norm{x^{\pi}_t} + \norm{u^{\pi}_t}) + c_t(0,0) \leq L_{c} ( b + (n+1)m) + c_t(0,0)$.  Hence, $c_t \leq \mathcal{O}(1)$ for all $t \in [1:T_0]$. Hence, summing over all $t \in [1:T_0]$, we get $
\sum_{t=1}^{T_0} c_t(x^{\pi}_t,u^{\pi}_t) \leq \mathcal{O}(T_0)$. 
Thus,
\beq
J_{1:T_{0}}(u^{\pi}_{1:T_{0}}; \theta^{\star}) -  J_{1:T_{0}}(u^{\pi^{\star}}_{1:T_{0}};\theta^{\star}) \leq \mathcal{O}(T_{0})  \nonumber 
\eeq
\end{proof}

\subsection{Proof of Proposition \ref{propo:TermII-o}}
\label{sec:proof-TermII-o}

Proposition \ref{propo:TermII-ce} and Proposition \ref{propo:TermII-o} characterizes the bound on Term II. Here, we only give the proof for Proposition \ref{propo:TermII-o} (for the O-MPC algorithm). The proof for Proposition \ref{propo:TermII-ce} (for the CE-MPC algorithm) is implied by this. 

\begin{proof}
Consider the system evolution $\tilde{x}_{t+1} = A^{\star}\tilde{x}_t + B^{\star}u^{\pi_{\mathrm{o}}}_t, ~\tilde{x}_{T_0+1} = \hat{x}_{T_0+1}$.  Let
\begin{align} 
& \tilde{x}^{\delta \theta}_t = \sum_{j=T_0+2}^{t} (A^\star)^{t-j}(\delta \theta_{j-1})\hat{z}_{j-1},~\text{where},~\delta \theta_j = (\theta^\star -\hat{\theta}_j),  \nonumber \\
& \hat{z}^\top_j = [\hat{x}^\top_j, (u^{\pi_{\mathrm{o}}}_j)^\top], ~ \tilde{x}^{\delta \theta}_{T_0+1} = 0.
\label{eq:TermII-o-Eq1}
\end{align}

We will show that  $\tilde{x}_{t} = \hat{x}_{t} + \tilde{x}^{\delta \theta}_t$ for all $t \geq T_0+1$ by mathematical induction. This trivially holds for $k = T_0+1$. Now, let $\tilde{x}_{k} = \hat{x}_{k} + \tilde{x}^{\delta \theta}_k$ be true for some $k \in [T_0+1:T]$. Then 
\begin{align}
\tilde{x}_{k+1} &= A^\star\tilde{x}_k +B^\star u^{\pi_{\mathrm{o}}}_k = A^\star(\hat{x}_{k} + \tilde{x}^{\delta \theta}_k) +B^\star u^{\pi_{\mathrm{o}}}_k \nonumber\\
& = \theta^\star \hat{z}_k + A^\star\tilde{x}^{\delta \theta}_{k} = \hat{\theta}_k \hat{z}_k + \delta \theta_k\hat{z}_k + A^\star\tilde{x}^{\delta \theta}_{k} \nonumber \\
& = \hat{x}_{k+1} + \tilde{x}^{\delta \theta}_{k+1}. 
\label{eq:TermII-o-Eq2}
\end{align}
This completes the induction argument. 

Similarly, we will show that $x^{\pi_{\mathrm{o}}}_t = \tilde{x}_t - (A^\star)^{t-T_0-1}\epsilon_{T_0+1}$ for all for all $t \geq T_0+1$ by mathematical induction. This holds at $t = T_0+1$, since $y_{T_0+1} = x_{T_0+1} + \epsilon_{T_0+1}$. Let $x^{\pi_{\mathrm{o}}}_k = \tilde{x}_k - (A^\star)^{k-T_0-1}(\epsilon_{T_0+1})$ hold for some $k \geq T_0+1$. Then 
\begin{align}
\label{eq:TermII-o-Eq22}
& x^{\pi_{\mathrm{o}}}_{k+1} = A^\star x^{\pi_{\mathrm{o}}}_k + B^\star u^{\pi_{\mathrm{o}}}_k \nonumber \\
& = A^\star (\tilde{x}_k - (A^\star)^{k-T_0-1}(\epsilon_{T_0+1})) + B^\star u^{\pi_{\mathrm{o}}}_k  \nonumber \\
& = A^\star \tilde{x}_k + B^\star u^{\pi_{\mathrm{o}}}_k - (A^\star)^{k+1-T_0-1}\epsilon_{T_0+1}. 
\end{align}
This completes the induction argument. 

Using Equations \ref{eq:TermII-o-Eq1} - \ref{eq:TermII-o-Eq22} we get
\begin{align}
\label{eq:prop-4-pf-step1}
& \sum_{t = T_0+1}^{T} c_t(x^{\pi_{\mathrm{o}}}_t, u^{\pi_{\mathrm{o}}}_t) \stackrel{(a)}{=} \sum_{t = T_0+1}^{T} c_t(\tilde{x}_t - (A^\star)^{t-T_0-1}\epsilon_{T_0+1}, u^{\pi_{\mathrm{o}}}_t)  \nonumber \\
& \stackrel{(b)}{=} \sum_{t = T_0+1}^{T} c_t(\hat{x}_t + \tilde{x}^{\delta \theta}_t - (A^\star)^{t-T_0-1}\epsilon_{T_0+1}, u^{\pi_{\mathrm{o}}}_t). 
\end{align}
Here, we get $(a)$ by writing $x^{\pi_{\mathrm{o}}}_t = \tilde{x}_t - (A^\star)^{t-T_0-1}\epsilon_{T_0+1}$ and $(b)$ by writing $\tilde{x}_{t} = \hat{x}_{t} + \tilde{x}^{\delta \theta}_t$.

According to Proposition \ref{prop:estimation}, $\theta^\star \in \widehat{\Theta}$, with high probability. We now claim that when $\theta^\star \in \widehat{\Theta}$, under the O-MPC algorithm,  $\max_{t \in [T_{0}+1, T]} \|\hat{x}_t\|$ is bounded and the bound does not depend on $T$. The proof for this claim is given as apart of the proof of  Proposition \ref{propo:TermIII-o}. 

The boundedness of  $\hat{x}_t$ implies the boundedness of  $u^{\pi_{\mathrm{o}}}_t$  since $u^{\pi_{\mathrm{o}}}_t$ is the solution of {O-MPC} Algorithm, whose solutions are continuous in $\hat{x}_t$ (for a proof, see \cite[Theorem 1.17]{rockafellar2009variational}). Then, it follows that $x^{\pi_{\mathrm{o}}}_t$ is bounded since $\rho(A^{\star}) < 1$ and $u^{\pi_{\mathrm{o}}}_t$ is bounded. The term  $(A^\star)^{t-T_0-1}\epsilon_{T_0+1}$ is also bounded, since $\rho(A^{\star}) < 1$ and $\epsilon_{T_0+1}$ is bounded.

For convenience, denote the bound on $\|\hat{x}_t\|$, $\|x^{\pi_{\mathrm{o}}}_t\|$ and $\|(A^\star)^{t-T_0-1}\epsilon_{T_0+1}\|$ as $b$ and the bound on $u^{\pi_{\mathrm{o}}}_t$ as $d$. Note that $b$ and $d$ are constants that do not change with the horizon length $T$. Let $\alpha_0$ be the local Lipschitz constant for all $c_t$s. Then, under the event $\theta^\star \in \widehat{\Theta}$, using Eq. \eqref{eq:prop-4-pf-step1} we get
\begin{align}
    \label{eq:prop-4-pf-step2}
    & \sum_{t = T_0+1}^{T} c_t(x^{\pi_{\mathrm{o}}}_t, u^{\pi_{\mathrm{o}}}_t) \leq \sum_{t = T_0+1}^{T} c_t(\hat{x}_t, u^{\pi_{\mathrm{o}}}_t) \nonumber \\
    & + \alpha_0  \sum_{t = T_0+1}^{T}  \left( \norm{\tilde{x}^{\delta \theta}_t} +\norm{(A^\star)^{t-T_0-1}\epsilon_{T_0+1}}\right).
\end{align}
Since $\rho(A^{\star}) < 1$, there exist $c_\rho > 0$ and $\lambda_\rho < 1$ such that $\norm{A^{k}} \leq c_\rho \lambda^{k}_\rho$. This implies that \begin{align}
    & \sum_{t = T_0+1}^{T} \norm{\tilde{x}^{\delta \theta}_t}  \stackrel{(c)}{=} \sum_{t = T_0+1}^{T} \norm{\sum_{j=T_0+1}^{t} (A^\star)^{t-j}(\delta \theta_{j-1})\hat{z}_{j-1}} \nonumber \\
    & {\leq} \sum_{t = T_0+1}^{T} \sum_{j=T_0+1}^{t} \norm{ (A^\star)^{t-j}(\delta \theta_{j-1})\hat{z}_{j-1}} \nonumber \\
    &{\leq} \sum_{t = T_0+1}^{T} \sum_{j=T_0+1}^{t} \norm{ (A^\star)^{t-j}} \norm{\delta \theta_{j-1}}\norm{\hat{z}_{j-1}} 
    \nonumber \\
    & \stackrel{(d)}{\leq} (b+d)\sum_{t = T_0+1}^{T} \sum_{j=T_0+1}^{t} \norm{ (A^\star)^{t-j}} \norm{\delta \theta_{j-1}} \nonumber \\
    & \stackrel{(e)}{\leq} 2(b+d)\beta(\delta) \sum_{t = T_0+1}^{T} \sum_{j=T_0+1}^{t} \norm{ (A^\star)^{t-j}} \nonumber \\
    & \stackrel{(f)}{\leq} 2(b+d)c_\rho \beta(\delta) \sum_{t = T_0+1}^{T} \sum_{j=T_0+1}^{t} \lambda^{t-j}_\rho \nonumber \\
    \label{eq:prop-4-pf-step3}
    &\leq \frac{2(b+d)c_\rho \beta(\delta)}{1-\lambda_\rho} T  \stackrel{(g)}{=}  \mathcal{O}\left(\frac{T}{\sqrt{T_0}}\right)
\end{align}
Here, we get $(c)$ by definition of $\tilde{x}^{\delta \theta}_t$,   $(d)$ by the bound on $\hat{x}_t$ and $u^{\pi_{\mathrm{o}}}_t$, $(e)$ by using the fact that  $\norm{\delta \theta_{j-1}} \leq 2\beta(\delta)$ by Proposition \ref{prop:estimation}, $(f)$ by using the fact that $\norm{(A^\star)^{k}} \leq c_\rho \lambda^k_\rho$, and $(b)$ by using the fact that $\beta(\delta) =  \mathcal{O}\left({1}/{\sqrt{T_0}}\right) $. 

Now,
\begin{align}
    \label{eq:prop-4-pf-step4}
    & \sum_{t = T_0+1}^{T} \norm{(A^\star)^{t-T_0-1}\epsilon_{T_0+1}} \leq \epsilon_c \sum_{t = T_0+1}^{T} c_\rho \lambda^{t-T_0-1}_\rho \nonumber \\
    & \leq \frac{\epsilon_c c_\rho}{1-\lambda_\rho} = \mathcal{O}(1).
\end{align}

Using \eqref{eq:prop-4-pf-step3} and \eqref{eq:prop-4-pf-step4} in \eqref{eq:prop-4-pf-step2}, we get, with probability greater than $1-\delta$, 
\beq 
\sum_{t = T_0+1}^{T} c_t(x^{\pi_{\mathrm{o}}}_t, u^{\pi_{\mathrm{o}}}_t) - \sum_{t = T_0+1}^{T} c_t(\hat{x}_t, u^{\pi_{\mathrm{o}}}_t) \leq \mathcal{O}\left(\frac{T}{\sqrt{T_0}}\right). \nonumber 
\eeq 
This implies that, with probability greater that $1-\delta$, 
\begin{align*}
 & J_{T_{0}+1:T}(u^{\pi_{\mathrm{o}}}_{T_{0}+1:T}; \theta^{\star}) -  J_{T_{0}+1:T}(u^{\pi_{\mathrm{o}}}_{T_{0}+1:T}; \hat{\theta}_{T_0+1:T}) \nonumber \\
 & \leq \mathcal{O}(\frac{T}{\sqrt{T_{0}}})
\end{align*}
\end{proof}

\subsection{Proofs of Proposition \ref{propo:TermIII-ce}}
\label{sec:proof-TermIII-ce}

\begin{proof}
Given the system with parameter $\theta = [A(\theta), B(\theta)]$, initial time $t$, initial state $x$, and control sequence $u_{t:T}$,  consider the  system evolution $x_{\tau+1} = A(\theta) x_{\tau} + B(\theta) u_{\tau}$, for $\tau \in [t, T]$ with $x_{t} = x$. Let $\phi_{t}(k, x, u_{t:T},\theta)$  denotes $x_{t+k}$ (the state at  time $t+k$) for any $k \geq 0$.  We will also denote this as  $\phi_{t}(k, x, u_{t:t+k-1}, \theta)$ since $u_{t+k:T}$ does not affect $x_{t+k}$. 

We will first show that $ J_{T_{0}+1:T}(u^{\pi_{\mathrm{ce}}}_{T_{0}+1:T}; \hat{\theta}_{\mathrm{ls}})$ is $\mathcal{O}(1)$  under the assumption stated. 
 
For any $t \in [T_{0}+1, T]$, consider the system evolution  $\hat{x}_{t+1} = A(\hat{\theta}_{\mathrm{ls}}) \hat{x}_{t}+B (\hat{\theta}_{\mathrm{ls}}) u^{\pi_{\mathrm{ce}}}_{t}$, with the initial state $\hat{x}_{T_{0}+1} = y_{T_{0}+1}$. Let
\begin{align*}
\hat{u}^{t}_{0:M-1} = &\argmin_{u_{t:t+M-1}} \sum^{t+M-1}_{k=t} c_{k}(x_k,u_k), \nonumber \\
& \text{s.t.}~~ x_{k+1} = A(\hat{\theta}_{\mathrm{ls}}) x_k+B(\hat{\theta}_{\mathrm{ls}}) u_k,~ x_{t} = \hat{x}_{t}.
\end{align*}
Recall from Eq. \eqref{eq:cost-to-go} that the optimal value of the above problem is ${V}_{t}(\hat{x}_{t}; \hat{\theta}_{\mathrm{ls}})$. Also, please note that according to the CE-MPC algorithm,   $u^{\pi_{\mathrm{ce}}}_{t} = \hat{u}^{t}_{0}$. Now, 
\begin{align}
&{V}_{t}(\hat{x}_{t}; \hat{\theta}_{\mathrm{ls}}) =  \sum_{k = 0}^{M-1} c_{t+k}(\phi_{t}(k, \hat{x}_{t}, \hat{u}^{t}_{0:k-1} ;\hat{\theta}_{\mathrm{ls}}), \hat{u}^{t}_{k} ) \nonumber \\
& \stackrel{(a)}{=}  c_{t}(\hat{x}_{t}, u^{\pi_{\mathrm{ce}}}_{t}) +  \sum_{k = 0}^{M-2} c_{t+k+1}(\phi_{t}(k+1, \hat{x}_{t}, \hat{u}^{t}_{0:k} ;\hat{\theta}_{\mathrm{ls}}) , \hat{u}^{t}_{k+1} ) \nonumber \\
& \stackrel{(b)}{=}  c_{t}(\hat{x}_{t}, u^{\pi_{\mathrm{ce}}}_{t} ) +  \sum_{k = 0}^{M-2} c_{t+k+1}(\phi_{t+1}(k, \hat{x}_{t+1}, \hat{u}^{t}_{1:k} ;\hat{\theta}_{\mathrm{ls}}) , \hat{u}^{t}_{k+1} ) \nonumber \\
&=  c_{t}(\hat{x}_{t}, u^{\pi_{\mathrm{ce}}}_{t} ) +  \sum_{k = 0}^{j-2} c_{t+k+1}(\phi_{t+1}(k, \hat{x}_{t+1}, \hat{u}^{t}_{1:k} ;\hat{\theta}_{\mathrm{ls}}) , \hat{u}^{t}_{k+1} ) \nonumber \\
& + \sum_{k = j-1 }^{M-2} c_{t+k+1}(\phi_{t+1}(k, \hat{x}_{t+1}, \hat{u}^{t}_{1:k} ;\hat{\theta}_{\mathrm{ls}}) , \hat{u}^{t}_{k+1} ), \nonumber \\
&=  c_{t}(\hat{x}_{t}, u^{\pi_{\mathrm{ce}}}_{t} ) +  \sum_{k = 0}^{j-2} c_{t+k+1}(\phi_{t+1}(k, \hat{x}_{t+1}, \hat{u}^{t}_{1:k} ;\hat{\theta}_{\mathrm{ls}}) , \hat{u}^{t}_{k+1} ) \nonumber \\
& + \sum_{k = 0 }^{M-j-1} c_{t+j+k}(\phi_{t+1}(j+k-1, \hat{x}_{t+1}, \hat{u}^{t}_{1:j+k-1} ;\hat{\theta}_{\mathrm{ls}}) , \hat{u}^{t}_{j+k} ), \nonumber \\
\label{eq:Vt-rewrite-pf}
&\stackrel{(c)}{=}  c_{t}(\hat{x}_{t}, u^{\pi_{\mathrm{ce}}}_{t} ) +  \sum_{k = 0}^{j-2} c_{t+k+1}(\phi_{t+1}(k, \hat{x}_{t+1}, \hat{u}^{t}_{1:k} ;\hat{\theta}_{\mathrm{ls}}) , \hat{u}^{t}_{k+1} )\nonumber\\
& + \sum_{k = 0 }^{M-j-1} c_{t+j+k}(\phi_{t+j}(k, \tilde{x}^{j}_{t+1}, \hat{u}^{t}_{j:j+k-1} ;\hat{\theta}_{\mathrm{ls}}) , \hat{u}^{t}_{j+k} ), 
\end{align}
Here, we get $(a)$ by using the fact that $\hat{u}^{t}_{0} = u^{\pi_{\mathrm{ce}}}_{t}$, $(b)$ by using the fact that $\hat{x}_{t+1} =  \phi_{t}(1, \hat{x}_{t}, \hat{u}^{t}_{0:0} ;\hat{\theta}_{\mathrm{ls}})$ and considering the summation from $t+1$ with initialization $\hat{x}_{t+1}$, and $(c)$  by denoting $\tilde{x}^{j}_{t+1} =  \phi_{t+1}(j-1, \hat{x}_{t+1}, \hat{u}^{t}_{1:j-1} ;\hat{\theta}_{\mathrm{ls}})$ and considering the sequence from time step $t+j$. 

Similarly,
\begin{align}
& {V}_{t+1}(\hat{x}_{t+1}; \hat{\theta}_{\mathrm{ls}})\nonumber \\
& =  \sum_{k = 0}^{M-1} c_{t+1+k}(\phi_{t+1}(k, \hat{x}_{t+1}, \hat{u}^{t+1}_{0:k-1} ; \hat{\theta}_{\mathrm{ls}}), \hat{u}^{t+1}_{k} ) \nonumber \\
&= \sum_{k = 0}^{j-2} c_{t+1+k}(\phi_{t+1}(k, \hat{x}_{t+1}, \hat{u}^{t+1}_{0:k-1} ;\hat{\theta}_{\mathrm{ls}}), \hat{u}^{t+1}_{k} ) \nonumber \\
& + \sum_{k = j-1}^{M-1} c_{t+1+k}(\phi_{t+1}(k, \hat{x}_{t+1}, \hat{u}^{t+1}_{0:k-1} ;\hat{\theta}_{\mathrm{ls}}), \hat{u}^{t+1}_{k} ) \nonumber\\
&= \sum_{k = 0}^{j-2} c_{t+1+k}(\phi_{t+1}(k, \hat{x}_{t+1}, \hat{u}^{t+1}_{0:k-1} ;\hat{\theta}_{\mathrm{ls}}), \hat{u}^{t+1}_{k} ) \nonumber \\
& + \sum_{k = 0}^{M-j} c_{t+j+k}(\phi_{t+1}(j+k-1, \hat{x}_{t+1}, \hat{u}^{t+1}_{0:j+k-2} ;\hat{\theta}_{\mathrm{ls}}), \hat{u}^{t+1}_{j+k-1} ) \nonumber\\
&\stackrel{(d)}{\leq}   \sum_{k = 0}^{j-2} c_{t+1+k}(\phi_{t+1}(k, \hat{x}_{t+1}, \hat{u}^{t}_{1:k} ;\hat{\theta}_{\mathrm{ls}}), \hat{u}^{t}_{k+1} ) \nonumber \\
& + \min_{\tilde{u}_{0:M-j}} \sum_{k = 0}^{M-j} c_{t+j+k}(\phi_{t+j}(k, \tilde{x}^{j}_{t+1}, \tilde{u}_{0:k-1} ;\hat{\theta}_{\mathrm{ls}}), \tilde{u}_{k} ) \nonumber\\
&\stackrel{(e)}{\leq}   \sum_{k = 0}^{j-2} c_{t+1+k}(\phi_{t+1}(k, \hat{x}_{t+1}, \hat{u}^{t}_{1:k} ;\hat{\theta}_{\mathrm{ls}}), \hat{u}^{t}_{k+1} ) + V_{t+j}(\tilde{x}^{j}_{t+1}; \hat{\theta}_{\mathrm{ls}}) \nonumber \\
\label{eq:Vtp1-rewrite-pf}
&\stackrel{(f)}{\leq}   \sum_{k = 0}^{j-2} c_{t+1+k}(\phi_{t+1}(k, \hat{x}_{t+1}, \hat{u}^{t}_{1:k} ;\hat{\theta}_{\mathrm{ls}}), \hat{u}^{t}_{k+1} ) \nonumber \\
& +  \overline{\alpha} \sigma(\tilde{x}^j_{t+1})
\end{align}
Here, we get $(d)$  by changing  the optimal sequence $\hat{u}^{t+1}_{0:j-2}$ to the control sequence $ \hat{u}^{t}_{1:j-1}$ in the first $j-1$ steps starting from time $(t+1)$ and finding the minimizing sequence for the remaining steps. For this, we denote  $\tilde{x}^{j}_{t+1} =  \phi_{t+1}(j-1, \hat{x}_{t+1}, \hat{u}^{t}_{1:j-1} ;\hat{\theta})$ as the initial state for the remaining summation. Please note that $ \tilde{x}^{j}_{t+1} $  we introduced in $(c)$ and $(d)$ above are indeed identical. We get  $(e)$ from the definition $V_{t+j}$ and $(f)$ from the premise of the proposition that $\hat{\theta}_{\mathrm{ls}}$ satisfies  the conditions given in Assumption \ref{as:stability} (in particular, Assumption \ref{as:stability} .$(ii)$)

Now, using \eqref{eq:Vt-rewrite-pf} and \eqref{eq:Vtp1-rewrite-pf}, 
\begin{align}
\label{eq:term-I-pf-step-1}
& {V}_{t+1}(\hat{x}_{t+1}; \hat{\theta}_{\mathrm{ls}})  - {V}_{t}(\hat{x}_{t}; \hat{\theta}_{\mathrm{ls}}) \nonumber \\
& \leq   \overline{\alpha} \sigma(\tilde{x}^j_{t+1}) -c_{t}(\hat{x}_{t}, u^{\pi_{\text{ce}}}_{t} ) \leq   \overline{\alpha} \sigma(\tilde{x}^j_{t+1}) - \underline{\alpha} \sigma(\hat{x}_{t}),
\end{align}
where the first inequality is obtained by  canceling the common terms and using the fact that $c_{t}$s are positive functions, and the second inequality by using Assumption \ref{as:stability}.$(i)$

For any $j \in [1, M-1]$,  using the fact that  $\tilde{x}^{j}_{t+1} = \phi_{t+1}(j-1, \hat{x}_{t+1}, \hat{u}^{t}_{1:j-1} ;\hat{\theta}_{\mathrm{ls}}) =  \phi_{t}(j, \hat{x}_{t}, \hat{u}^{t}_{0:j-1} ;\hat{\theta}_{\mathrm{ls}})$, Assumption \ref{as:stability} .$(i)$, and Assumption \ref{as:stability} .$(ii)$, we get
\begin{align*}
& \underline{\alpha} \sum^{M-1}_{j=1} \sigma(\tilde{x}^j_{t+1}) \leq  \sum^{M-1}_{k=0} c_{t+k}( \phi_{t}(k, \hat{x}_{t}, \hat{u}^{t}_{0:k-1} ;\hat{\theta}_{\mathrm{ls}}), \hat{u}^{t}_{k}) \nonumber \\
& = {V}_{t}(\hat{x}_{t}; \hat{\theta}_{\mathrm{ls}})  \leq  \overline{\alpha} \sigma(\hat{x}_{t}).
\end{align*}
Hence, there exists $j^{*} \in [1, M-1]$ such that
\begin{align}
\label{eq:term-I-pf-step-2}
\sigma(\tilde{x}^{j^{*}}_{t+1}) \leq \frac{(\overline{\alpha}/\underline{\alpha})}{(M-1)} \sigma(\hat{x}_t).
\end{align}
Using \eqref{eq:term-I-pf-step-2} in \eqref{eq:term-I-pf-step-1}, for $j = j^{*}$, we get
\begin{align}
{V}_{t+1}(\hat{x}_{t+1}; \hat{\theta}_{\mathrm{ls}})  - {V}_{t}(\hat{x}_{t}; \hat{\theta}_{\mathrm{ls}})     \leq (\frac{(\overline{\alpha}/\underline{\alpha})^2}{(M-1)} - 1) \underline{\alpha} \sigma(\hat{x}_t)
\end{align}
Let $\gamma = \frac{(\overline{\alpha}/\underline{\alpha})^2}{(M-1)}$. Note that since $M > (\overline{\alpha}/\underline{\alpha})^2+1$ according to the premise of the proposition, we have $\gamma < 1$. Using this, we get 
\begin{align*}
& {V}_{t+1}(\hat{x}_{t+1}; \hat{\theta}_{\mathrm{ls}})  - {V}_{t}(\hat{x}_{t}; \hat{\theta}_{\mathrm{ls}})  \leq - (1 - \gamma) \underline{\alpha} \sigma(\hat{x}_t) \nonumber \\
& =  - (1 - \gamma) ({\underline{\alpha}}/{\overline{\alpha}}) \overline{\alpha} \sigma(\hat{x}_t)  \leq  - (1 - \gamma) ({\underline{\alpha}}/{\overline{\alpha}})  {V}_{t}(\hat{x}_{t}; \hat{\theta}_{\mathrm{ls}}) ,
\end{align*}
where we used Assumption \ref{as:stability} .$(ii)$ to get the last inequality. 
This will yield, 
\begin{align*}
{V}_{t+1}(\hat{x}_{t+1}; \hat{\theta}_{\mathrm{ls}})  \leq \bar{\gamma}  {V}_{t}(\hat{x}_{t}; \hat{\theta}_{\mathrm{ls}}) ,
\end{align*}
where $\bar{\gamma}  = 1 - (1 - \gamma) ({\underline{\alpha}}/{\overline{\alpha}})$. Applying the above inequality repeatedly, we get
\begin{align*}
& c_{t}(\hat{x}_{t},  u^{\pi_{\mathrm{ce}}}_{t}) \leq {V}_{t}(\hat{x}_{t}; \hat{\theta}_{\mathrm{ls}})  \leq \gamma {V}_{t-1}(\hat{x}_{t-1}; \hat{\theta}_{\mathrm{ls}})  \leq \ldots \nonumber \\
& \leq \gamma^{t-T_{0}-1}  {V}_{T_{0}+1}(\hat{x}_{T_{0}+1}; \hat{\theta}_{\mathrm{ls}}) \leq  \gamma^{t-T_{0}-1} \overline{\alpha} \sigma(\hat{x}_{T_{0}+1}). 
\end{align*}
Taking summation on both sides,
\begin{align*}
\sum^{T}_{t=T_{0}+1} c_{t}(\hat{x}_{t},  u^{\pi_{\mathrm{ce}}}_{t})  \leq \frac{\overline{\alpha}}{(1-\gamma)} \sigma(\hat{x}_{T_{0}+1}) = \mathcal{O}(1).
\end{align*}

This implies that  \beq J_{T_{0}+1:T}(u^{\pi_{\mathrm{ce}}}_{T_{0}+1:T}; \hat{\theta}_{\mathrm{ls}}) = \sum^{T}_{t=T_{0}+1} c_{t}(\hat{x}_{t}, u^{\pi_{\mathrm{ce}}}_{t}) = \mathcal{O}(1). \nonumber  \eeq 

This will also imply that $J_{T_{0}+1:T}(u^{\pi_{\mathrm{ce}}}_{T_{0}+1:T}; \hat{\theta}_{\mathrm{ls}}) -  J_{T_{0}+1:T}(u^{\pi^{\star}}_{T_{0}+1:T};\theta^{\star}) \leq  \mathcal{O}(1)$, which concludes the proof.
\end{proof}

\subsection{Proof of the Proposition \ref{propo:TermIII-o}}
\label{sec:proof-TermIII-o}

\begin{proof}
Given the system with parameter $\theta = [A(\theta), B(\theta)]$, initial time $t$, initial state $x$, and control sequence $u_{t:T}$,  consider the  system evolution $x_{\tau+1} = A(\theta) x_{\tau} + B(\theta) u_{\tau}$, for $\tau \in [t, T]$ with $x_{t} = x$. Let $\phi_{t}(k, x, u_{t:T},\theta)$  denotes $x_{t+k}$ (the state at  time $t+k$) for any $k \geq 0$.  We will also denote this as  $\phi_{t}(k, x, u_{t:t+k-1}, \theta)$ since $u_{t+k:T}$ does not affect $x_{t+k}$. 

We will first show that $J_{T_{0}+1:T}(u^{\pi_{\mathrm{o}}}_{T_{0}+1:T}; \hat{\theta}_{T_0+1:T})$ is $\mathcal{O}(T/\sqrt{T_0})$ under the stated assumptions. 

For any $t \in [T_{0}+1, T]$, consider the system evolution $\hat{x}_{t+1} = A(\hat{\theta}_{t}) \hat{x}_{t}+B (\hat{\theta}_{t}) u^{\pi_{\mathrm{o}}}_{t}$, with the initial state $\hat{x}_{T_{0}+1} = y_{T_{0}+1}$. Let
\begin{align}
(\hat{u}^{t}_{0:M-1}(\hat{\theta}_t),&  \hat{\theta}_t) =  \argmin_{u_{t:t+M-1}, \hat{\theta} \in \hat{\Theta}} \sum^{t+M-1}_{k=t} c_{k}(x_k,u_k), \nonumber \\
& \text{s.t.}~~ x_{k+1} = A(\hat{\theta}) x_k+B(\hat{\theta}) u_k,~ x_{t} = \hat{x}_{t}.
\label{eq:o-rhc-opt}
\end{align}

Let $\hat{x}^j_t = \phi_{t}(j, \hat{x}_{t}, \hat{u}^{t}_{0:j-1}(\hat{\theta}_t);\theta^{\star})$, and $\tilde{x}^j_t = \phi_{t+1}(j-1, \hat{x}_{t+1}, \hat{u}^{t}_{1:j-1}(\hat{\theta}_t) ;\theta^\star)$ for all $j \in [1,M]$. Also, define
\beq 
\tilde{u}^j_{0:M-1} = \argmin_{\tilde{u}_{0:M-1}} \sum_{k = 0}^{M-1}c_{t+j+k}(\phi_{t+j}(k, \hat{x}^j_{t}, \tilde{u}_{0:k-1};\theta^{\star}), \tilde{u}_k). \nonumber 
\eeq 

Now, by definition 
\begin{align}
& \hat{x}_{t+1} = A(\hat{\theta}_t)\hat{x}_t +B(\hat{\theta}_t)\hat{u}^t_0(\hat{\theta}_t) = (A(\hat{\theta}_t) - A(\theta^\star))\hat{x}_t \nonumber \\
& + (B(\hat{\theta}_t) - B(\theta^\star))\hat{u}^t_0(\hat{\theta}_t) + A(\theta^\star)\hat{x}_t + B(\theta^\star)\hat{u}^t_0(\hat{\theta}_t) \nonumber \\
& = \hat{x}^1_t - \delta\theta_t\hat{z}_t, ~ ~ \delta \theta_t = (\theta^{\star} -\hat{\theta}_t), ~\hat{z}^\top_t = \left[\hat{x}^\top_t, (\hat{u}^t_0(\hat{\theta}_t))^\top \right]. \nonumber 
\end{align}
Hence, by definition, $\forall ~k, ~ k \in [0,M-1]$,
\begin{align}
& \phi_{t+1}(k, \hat{x}_{t+1}, \hat{u}^t_{1:k}(\hat{\theta}_t);\theta^\star) - \phi_{t+1}(k, \hat{x}^1_{t}, \hat{u}^t_{1:k}(\hat{\theta}_t);\theta^\star)\nonumber \\
& = -A(\theta^\star)^{k}\delta\theta_t\hat{z}_t. 
\label{eq:TermIII-o-Eq1}
\end{align}

Consider $x_{k+1} = A(\theta^{\star})x_k + B(\theta^{\star})\hat{u}^t_{k-t}(\hat{\theta}_t), ~ k \in [t,t+M-1], ~x_t = \hat{x}_t$, and $\tilde{x}_{k+1} = A(\hat{\theta}_t)\tilde{x}_k + B(\hat{\theta}_t)\hat{u}^t_{k-t}(\hat{\theta}_t), ~\tilde{x}_t = \hat{x}_t$. Then, applying the same argument from Eq. \eqref{eq:TermII-o-Eq1} to \eqref{eq:TermII-o-Eq2} in the proof of Proposition \ref{propo:TermII-o}, we get that
\begin{align}
& x_{k} = \tilde{x}_{k} + x^{\delta \theta}_k, ~ \text{where,} ~ x^{\delta \theta}_k = \sum_{j=t+1}^{k} (A^\star)^{k-j}(\delta \theta_{t})\tilde{z}_{j-1}, \tilde{z}^\top_{j} \nonumber \\
& = [\tilde{x}^\top_{j},\hat{u}^t_{j-t}(\hat{\theta}_t)], ~x^{\delta \theta}_t = 0,~ \delta \theta_t = (\theta^{\star} -\hat{\theta}_t). 
\label{eq:TermIII-o-Eq2}
\end{align}

Since the solution to Eq. \eqref{eq:o-rhc-opt} is continuous in $\hat{x}_t$ (for proof, see assertion (c) of \cite[Theorem 1.17]{rockafellar2009variational}), $\hat{u}^t_{0:M-1}$ and $\hat{\theta}_t$ are continuous functions of $\hat{x}_t$. Then, it follows that $x_k, \tilde{x}_k, \hat{x}^{k+1-t}_t, \tilde{x}^{k+1-t}_t$ and $x^{\delta \theta}_k$ in Eq. \eqref{eq:TermIII-o-Eq2} are also continuous functions of $\hat{x}_t$ for all $k \in [t,t+M-1]$. Similarly, $\tilde{u}^{k+1-t}_{0:M-1}$ is a continuous function of $\hat{x}_t$ for all $k \in [t,t+M-1]$, because it is a continuous function of $\hat{x}^{k+1-t}_t$. Therefore $\phi_{k+1}(l, \hat{x}^{k+1-t}_{t}, \tilde{u}^{k+1-t}_{0:l-1};\theta^{\star})$ is a continuous function of $\hat{x}_t$ for all $l \in [0,M-1], ~ k \in [t,t+M-1]$.

Next, we prove by induction that, under the event $\theta^\star \in \widehat{\Theta}$, $\hat{x}_t$ is bounded by a constant that does not increase with $T_0$ and $T$. 
Let's assume that the bound on $\hat{x}_t$ under the event $\theta^\star \in \widehat{\Theta}$ to be the constant $b$. Then, there exist functions $\overline{b}$ and $d$ with $dom(\overline{b}) = dom(d) = \mathbb{R}$ such that, the bound on $x_k, \tilde{x}_k, \hat{x}^{k+1-t}_t,  \tilde{x}^{k+1-t}_t, x^{\delta \theta}_k$ and $\phi_{k+1}(l, \hat{x}^{k+1-t}_{t}, \tilde{u}^{k+1-t}_{0:l-1};\theta^{\star})$ for all $k \in [t,t+M-1]$ is $\overline{b}(b)$ and the bound on $\tilde{u}^{k+1-t}$ and $\hat{u}^t_{k-t}$ for all $k \in [t,t+M-1]$ is $d(b)$. Let $\alpha_0$ be the local Lipschitz constant for $c_t$s in the compact set $\{(x,u): \norm{x} \leq \overline{b}(b), \norm{u} \leq d(b)\}$.

Recall from \eqref{eq:cost-to-go} that the optimal value of the above problem is ${V}_{t}(\hat{x}_{t}; \hat{\theta}_{t})$. Since $\rho(A^\star) < 1$, there exist constants $c_\rho > 0$ and $\lambda_\rho < 1$ such that $\norm{A^k} \leq c_\rho \lambda^k_\rho$. Also, please note that according to the O-MPC algorithm, $u^{\pi_{\mathrm{o}}}_{t} = \hat{u}^{t}_{0}$. Now, under the event $\theta^\star \in \widehat{\Theta}$,
\begin{align}
& V_{t+1}(\hat{x}_{t+1};\hat{\theta}_{t+1}) \nonumber \\
& = \sum_{k = 0}^{M-1} c_{t+1+k}(\phi_{t+1}(k, \hat{x}_{t+1}, \hat{u}^{t+1}_{0:k-1}(\hat{\theta}_{t+1});\hat{\theta}_{t+1}), \hat{u}^t_{k}(\hat{\theta}_{t+1})) \nonumber \\
& \stackrel{(a)}{\leq} \sum_{k = 0}^{M-1} c_{t+1+k}(\phi_{t+1}(k, \hat{x}_{t+1}, \hat{u}^{t+1}_{0:k-1}(\theta^{\star});\theta^{\star}), \hat{u}^{t+1}_{k}(\theta^{\star})) \nonumber \\
& \stackrel{(b)}{\leq} \sum_{k = 0}^{j-2} c_{t+1+k}(\phi_{t+1}(k, \hat{x}_{t+1}, \hat{u}^{t}_{1:k}(\hat{\theta}_t);\theta^{\star}), \hat{u}^t_{k+1}(\hat{\theta}_t))\nonumber \\
& + \sum_{k = 0}^{M-j} c_{t+j+k}(\phi_{t+j}(k, \tilde{x}^j_t, \tilde{u}^j_{0:k-1};\theta^{\star}), \tilde{u}^j_k) \nonumber \\
& \stackrel{(c)}{\leq} \sum_{k = 0}^{j-2} c_{t+1+k}(\phi_{t+1}(k, \hat{x}^1_{t}, \hat{u}^{t}_{1:k}(\hat{\theta}_t);\theta^{\star}), \hat{u}^t_{k+1}(\hat{\theta}_t)) \nonumber \\
& + \sum_{k = 0}^{M-j} c_{t+j+k}(\phi_{t+j}(k, \hat{x}^j_{t}, \tilde{u}^j_{0:k-1};\theta^{\star}), \tilde{u}^j_k) \nonumber \\
& + \alpha_0\sum_{k = 0}^{M-1} \norm{A(\theta^\star)^{k}\delta\theta_t\hat{z}_t} \nonumber \\
& \leq \sum_{k = 0}^{j-2} c_{t+1+k}(\phi_{t+1}(k, \hat{x}^1_{t}, \hat{u}^{t}_{1:k}(\hat{\theta}_t);\theta^{\star}), \hat{u}^t_{k+1}(\hat{\theta}_t)) \nonumber \\
& + \sum_{k = 0}^{M-1} c_{t+j+k}(\phi_{t+j}(k, \hat{x}^j_{t}, \tilde{u}^j_{0:k-1};\theta^{\star}), \tilde{u}^j_k) \nonumber \\
& + \alpha_0\sum_{k = 0}^{M-1} \norm{A(\theta^\star)^{k}\delta\theta_t\hat{z}_t} \nonumber \\
& \stackrel{(d)}{\leq} \sum_{k = 0}^{j-2} c_{t+1+k}(\phi_{t+1}(k, \hat{x}^1_{t}, \hat{u}^{t}_{1:k}(\hat{\theta}_t);\theta^{\star}), \hat{u}^t_{k+1}(\hat{\theta}_t)) \nonumber \\
& + V_{t+j}(\hat{x}^j_{t};\theta^{\star}) + \alpha_0\sum_{k = 0}^{M-1} \norm{A(\theta^\star)^{k}\delta\theta_t\hat{z}_t}  \nonumber \\
& \stackrel{(e)}{\leq} \sum_{k = 0}^{j-2} c_{t+1+k}(\phi_{t+1}(k, \hat{x}^1_{t}, \hat{u}^{t}_{1:k}(\hat{\theta}_t);\theta^\star), \hat{u}^t_{k+1}(\hat{\theta}_t)) \nonumber \\
& + \overline{\alpha} \sigma(\hat{x}^j_{t}) + \alpha_0\sum_{k = 0}^{M-1} \norm{A(\theta^\star)^{k}\delta\theta_t\hat{z}_t} \nonumber \\
& \stackrel{(f)}{\leq} \sum_{k = 0}^{j-2} c_{t+1+k}(\phi_{t+1}(k, \hat{x}^1_{t}, \hat{u}^{t}_{1:k}(\hat{\theta}_t);\theta^\star), \hat{u}^t_{k+1}(\hat{\theta}_t)) \nonumber \\
& + \overline{\alpha} \sigma(\hat{x}^j_{t}) + \alpha_0(\overline{b}(b)+d(b))\norm{\delta\theta_t}\sum_{k = 0}^{M-1} \norm{A(\theta^\star)^{k}} \nonumber \\
& \leq \sum_{k = 0}^{j-2} c_{t+1+k}(\phi_{t+1}(k, \hat{x}^1_{t}, \hat{u}^{t}_{1:k}(\hat{\theta}_t);\theta^\star), \hat{u}^t_{k+1}(\hat{\theta}_t)) \nonumber \\
& + \overline{\alpha} \sigma(\hat{x}^j_{t}) + \frac{\alpha_0c_\rho(\overline{b}(b)+d(b))\norm{\delta\theta_t}}{1-\lambda_\rho}. 
\end{align}
Here, we get $(a)$ by the fact that $\hat{\theta}_t$ is the optimal model at $t$, $(b)$ by the fact that the sequence of actions $\hat{u}^t_{1:j-1}(\hat{\theta}_t)$ for the first $j-1$ steps followed by $\tilde{u}^j_{0:M-j}$ is suboptimal to $\hat{u}^{t+1}_{0:M-1}(\theta)$, $(c)$ by using Eq. \eqref{eq:TermIII-o-Eq1} and Lipschitz condition on $c_t$s, $(d)$ by Eq. \eqref{eq:cost-to-go}, $(e)$ by Assumption \ref{as:stability}.$(ii)$ and $(f)$ by applying Cauchy-Schwarz to the last term. Similarly for $V_{t}(\hat{x}_{t};\hat{\theta}_{t})$,
\begin{align}
& V_{t}(\hat{x}_{t};\hat{\theta}_{t}) \nonumber \\
& = \sum_{k = 0}^{M-1} c_{t+k}(\phi_t(k, \hat{x}_t, \hat{u}^t_{0:k-1}(\hat{\theta}_t);\hat{\theta}_t), \hat{u}^t_{k}(\hat{\theta}_t)) \nonumber \\
& = \sum_{k = 0}^{M-1} c_{t+k}(\tilde{x}_{t+k}, \hat{u}^t_{k}(\hat{\theta}_t)) = \sum_{k = 0}^{M-1} c_{t+k}(x_{t+k} - x^{\delta \theta}_{t+k}, \hat{u}^t_{k}(\hat{\theta}_t)) \nonumber \\
& \stackrel{(g)}{\geq} \sum_{k = 0}^{M-1} c_{t+k}(x_{t+k}, \hat{u}^t_{k}(\hat{\theta}_t)) - \alpha_0 \sum_{k = 0}^{M-1} \norm{x^{\delta \theta}_{t+k}} \nonumber\\
& = \sum_{k = 0}^{M-1} c_{t+k}(x_{t+k}, \hat{u}^t_{k}(\hat{\theta}_t))\nonumber \\
& - \alpha_0 \sum_{k = 0}^{M-1} \norm{\sum_{j=t+1}^{t+k} (A^\star)^{t+k-j}(\delta \theta_{t})\tilde{z}_{j-1}} \nonumber \\
& \stackrel{(h)}{\geq} \sum_{k = 0}^{M-1} c_{t+k}(x_{t+k}, \hat{u}^t_{k}(\hat{\theta}_t)) \nonumber \\
& - \alpha_0 \sum_{k = 0}^{M-1}\sum_{j=t+1}^{t+k} \norm{(A^\star)^{t+k-j}}\norm{\delta \theta_{t}}\norm{\tilde{z}_{j-1}} \nonumber \\
& \geq \sum_{k = 0}^{M-1} c_{t+k}(x_{t+k}, \hat{u}^t_{k}(\hat{\theta}_t)) \nonumber \\
& - \alpha_0(\overline{b}(b)+d(b))\norm{\delta \theta_{t}}\sum_{k = 0}^{M-1}\sum_{j=t+1}^{t+k} \norm{(A^\star)^{t+k-j}} \nonumber \\
& \stackrel{(i)}{\geq} \sum_{k = 0}^{M-1} c_{t+k}(x_{t+k}, \hat{u}^t_{k}(\hat{\theta}_t)) - \frac{\alpha_0(\overline{b}(b)+d(b))c_\rho\norm{\delta \theta_{t}}M}{1 - \lambda_\rho} \nonumber \\
& \stackrel{(j)}{\geq} c_t(\hat{x}_t, u^{\pi_{\mathrm{o}}}_{t}) \nonumber \\
& + \sum_{k = 0}^{j-2} c_{t+1+k}(\phi_{t+1}(k, \hat{x}^1_{t}, \hat{u}^{t}_{1:k}(\hat{\theta}_t);\theta^{\star}), \hat{u}^{t}_{k+1}(\hat{\theta}_t)) \nonumber \\
& - \frac{\alpha_0(\overline{b}(b)+d(b))c_\rho\norm{\delta \theta_{t}}M}{1 - \lambda_\rho}.
\end{align}
Here, we get $(g)$ by applying the local Lipschitz condition for $c_t$s, $(h)$ by applying triangle inequality for norms and then applying Cauchy-Schwarz inequality, $(i)$ by using $\norm{(A^\star)^{k}} \leq c_\rho \lambda^k_\rho$ and summing over $j$ and $k$, $(j)$ by retaining the first $j-1$ terms in the first sum, using $u^{\pi_{\mathrm{o}}}_{t} = \hat{u}^{t}_{0}(\hat{\theta}_t)$, and the definition of $\hat{x}^1_t$ and $x_{t+1+k} = \phi_{t+1}(k, \hat{x}^1_{t}, \hat{u}^{t}_{1:k}(\hat{\theta}_t);\theta^{\star})$.

Now taking the difference between $V_{t+1}(\hat{x}_{t+1};\hat{\theta}_{t+1})$ and $V_{t}(\hat{x}_{t};\hat{\theta}_{t})$, 
\begin{align}
    & V_{t+1}(\hat{x}_{t+1};\hat{\theta}_{t+1}) - V_{t}(\hat{x}_{t};\hat{\theta}_{t})  \leq \overline{\alpha} \sigma(\hat{x}^j_{t}) \nonumber \\
    & - c_t(\hat{x}_t, u^{\pi_{\mathrm{o}}}_{t}) + \frac{\alpha_0c_\rho(\overline{b}(b)+d(b))\norm{\delta\theta_t}(M+1)}{1-\lambda_\rho}. \label{eq:Vdiff-o-rhc} 
\end{align}

By Assumption \ref{as:stability}.$(i)$, for all $j \in [1,M-1]$, $
\underline{\alpha}\sigma(\hat{x}^j_{t}) \leq c_{t+j}(\hat{x}^j_t, \hat{u}^t_{j}(\hat{\theta}_t))$. Then, summing over $j$ and adding $c_t(\hat{x}_t,\hat{u}^t_0(\hat{\theta}_t))$ to the right, we get
\begin{align} 
& \sum_{j=1}^{M-1} \underline{\alpha}\sigma(\hat{x}^j_{t}) \leq c_t(\hat{x}_t,\hat{u}^t_0(\hat{\theta}_t))  + \sum_{j=1}^{M-1}  c_{t+j}(\hat{x}^j_t, \hat{u}^t_{j}(\hat{\theta}_t)) \nonumber \\
& = c_t(\hat{x}_t,\hat{u}^t_0(\hat{\theta}_t)) \nonumber \\
& + \sum_{j=1}^{M-1}  c_{t+j}(\phi_{t}(j, \hat{x}_{t}, \hat{u}^{t}_{0:j-1}(\hat{\theta}_t);\theta^{\star}), \hat{u}^t_{j}(\hat{\theta}_t)) \nonumber \\
& = \sum_{j=0}^{M-1} c_{t+j}(\phi_{t}(j, \hat{x}_{t}, \hat{u}^{t}_{0:j-1}(\hat{\theta}_t);\theta^{\star}), \hat{u}^t_{j}(\hat{\theta}_t)) \nonumber \\
& \stackrel{(k)}{\leq} \sum_{j=0}^{M-1} c_{t+j}(\phi_{t}(j, \hat{x}_{t}, \hat{u}^{t}_{0:j-1}(\hat{\theta}_t);\hat{\theta}_t), \hat{u}^t_{j}(\hat{\theta}_t)) \nonumber \\
& + \alpha_0 \sum_{k = 0}^{M-1} \norm{x^{\delta \theta}_{t+k}} \nonumber \\
& \stackrel{(l)}{\leq} \sum_{j=0}^{M-1} c_{t+j}(\phi_{t}(j, \hat{x}_{t}, \hat{u}^{t}_{0:j-1}(\hat{\theta}_t);\hat{\theta}_t), \hat{u}^t_{j}(\hat{\theta}_t)) \nonumber \\
& + \frac{\alpha_0(\overline{b}(b)+d(b))c_\rho\norm{\delta \theta_{t}}M}{1 - \lambda_\rho} \nonumber \\
& \stackrel{(m)}{\leq} \sum_{j=0}^{M-1} c_{t+j}(\phi_{t}(j, \hat{x}_{t}, \hat{u}^{t}_{0:j-1}(\theta^\star);\theta^\star), \hat{u}^t_{j}(\theta^\star)) \nonumber \\
& + \frac{\alpha_0(\overline{b}(b)+d(b))c_\rho\norm{\delta \theta_{t}}M}{1 - \lambda_\rho} \nonumber \\
& \stackrel{(n)}{\leq} \overline{\alpha} \sigma(\hat{x}_t) + \frac{\alpha_0(\overline{b}(b)+d(b))c_\rho\norm{\delta \theta_{t}}M}{1 - \lambda_\rho} 
\end{align}
Here, we get $(k)$ by $(g)$, $(l)$ by computing the sum $\alpha_0 \sum_{k = 0}^{M-1} \norm{x^{\delta \theta}_{t+k}}$ as in steps $(g)$ to $(i)$, $(m)$ by the fact that $\hat{\theta}_t$ is the optimal model at $t$ and finally $(n)$ by Assumption \ref{as:stability}.$(ii)$. This implies that, there exists $j^{*} \in [1,M-1]$ such that
\beq 
\sigma(\hat{x}^{j^{*}}_{t}) \leq \frac{\overline{\alpha} \sigma(\hat{x}_t)}{\underline{\alpha}(M-1)} + \frac{\alpha_0(\overline{b}(b)+d(b))c_\rho\norm{\delta \theta_{t}}M}{\underline{\alpha}(M-1)(1 - \lambda_\rho)}. \label{eq:o-rhc-jstar}
\eeq 

Then, setting $j = j^{*}$ in Eq. \eqref{eq:Vdiff-o-rhc}, we get
\begin{align}
& V_{t+1}(\hat{x}_{t+1};\hat{\theta}_{t+1}) - V_{t}(\hat{x}_{t};\hat{\theta}_{t}) \leq \overline{\alpha} \sigma(\hat{x}^{j^{*}}_{t}) \nonumber \\
& - c_t(\hat{x}_t, u^{\pi_{\mathrm{o}}}_{t}) + \frac{\alpha_0c_\rho(\overline{b}(b)+d(b))\norm{\delta\theta_t}(M+1)}{1-\lambda_\rho} \nonumber \\
& \stackrel{(o)}{\leq} \frac{\overline{\alpha}^2 \sigma(\hat{x}_t)}{\underline{\alpha}(M-1)} - c_t(\hat{x}_t, u^{\pi_{\mathrm{o}}}_{t}) +  \left(\frac{\alpha_0c_\rho(\overline{b}(b)+d(b))(M+1)}{1-\lambda_\rho} \right. \nonumber \\
& \left. + \frac{\overline{\alpha}\alpha_0(\overline{b}(b)+d(b))c_\rho M}{\underline{\alpha}(M-1)(1 - \lambda_\rho)}  \right) \norm{\delta\theta_t} \stackrel{(p)}{=} \frac{\overline{\alpha}^2 \sigma(\hat{x}_t)}{\underline{\alpha}(M-1)} - c_t(\hat{x}_t, u^{\pi_{\mathrm{o}}}_{t})\nonumber \\
& + \left(\frac{\alpha_0c_\rho(\overline{b}(b)+d(b))(M+1)}{1-\lambda_\rho} \right. \nonumber \\
& \left. + \frac{\overline{\alpha}\alpha_0(\overline{b}(b)+d(b))c_\rho M}{\underline{\alpha}(M-1)(1 - \lambda_\rho)}  \right)\norm{\delta\theta_t} \nonumber \\
& \stackrel{(q)}{\leq} \left(\frac{\overline{\alpha}^2}{\underline{\alpha}^2(M-1)}-1\right)\underline{\alpha}\sigma(\hat{x}_t) + \left(\frac{\alpha_0c_\rho(\overline{b}(b)+d(b))(M+1)}{1-\lambda_\rho} \right. \nonumber \\
& \left. + \frac{\overline{\alpha}\alpha_0(\overline{b}(b)+d(b))c_\rho M}{\underline{\alpha}(M-1)(1 - \lambda_\rho)}  \right)\norm{\delta\theta_t}. \nonumber
\end{align} 
Here, we get $(o)$ by Eq. \eqref{eq:o-rhc-jstar}, $(p)$ by the fact that the factors accompanying $\norm{\delta\theta_t}$ are all constants, $(q)$ by Assumption \ref{as:stability}.$(i)$.

Let $\gamma = \frac{\overline{\alpha}^2}{\underline{\alpha}^2(M-1)}$. Since $M > (\overline{\alpha}/\underline{\alpha})^2+1$, we have $\gamma < 1$. Let $g(b,M) = \left(\frac{\alpha_0c_\rho(\overline{b}(b)+d(b))(M+1)}{1-\lambda_\rho} + \frac{\overline{\alpha}\alpha_0(\overline{b}(b)+d(b))c_\rho M}{\underline{\alpha}(M-1)(1 - \lambda_\rho)}  \right)$. Using these observations, we get
\begin{align}
 & V_{t+1}(\hat{x}_{t+1};\hat{\theta}_{t+1}) - V_{t}(\hat{x}_{t};\hat{\theta}_{t}) \leq -\left(1-\gamma\right)\underline{\alpha}\sigma(\hat{x}_t) \nonumber \\
 & + g(b,M)\norm{\delta\theta_t}  \leq -\left(1-\gamma\right)\frac{\underline{\alpha}}{\overline{\alpha}}V_{t}(\hat{x}_{t};\hat{\theta}_{t}) + g(b,M)\norm{\delta\theta_t}, \nonumber 
\end{align} 
where we used Assumption \ref{as:stability}.$(ii)$ to get the last inequality. This yields
\beq 
 V_{t+1}(\hat{x}_{t+1};\hat{\theta}_{t+1}) \leq \overline{\gamma}V_{t}(\hat{x}_{t};\hat{\theta}_{t}) + g(b,M)\norm{\delta\theta_t}, \nonumber 
\eeq 

$\overline{\gamma} = 1-\left(1-\gamma\right)\frac{\underline{\alpha}}{\overline{\alpha}}$. Under the event $\theta^\star \in \widehat{\Theta}$, $\norm{\delta \theta_t} \leq 2\beta(\delta)$ for all $t \geq T_0+1$. Therefore,
\beq 
 V_{t+1}(\hat{x}_{t+1};\hat{\theta}_{t+1}) \leq \overline{\gamma}V_{t}(\hat{x}_{t};\hat{\theta}_{t}) + \mathcal{O}(\frac{g(b,M)}{\sqrt{T_0}}). 
 \label{eq:Vt-bound-o-rhc}
\eeq

Now, by Assumption \ref{as:stability}$.(ii)$, and that $\hat{\theta}_t$ is optimistic, $V_{t}(\hat{x}_{t};\hat{\theta}_{t}) \leq \overline{\alpha}\sigma(\hat{x}_t)$, for all $t$. Also, by Assumption \ref{as:stability}.$(i)$, $V_t(x;\theta) \geq \underline{\alpha} \sigma(x)$ for any $\theta$. Hence, from Eq. \eqref{eq:Vt-bound-o-rhc}, we get
\beq
\sigma(\hat{x}_{t+1}) \leq \overline{\gamma}\overline{\alpha}/\underline{\alpha} \sigma(\hat{x}_{t}) + \mathcal{O}(\frac{g(b,M)}{\underline{\alpha}\sqrt{T_0}}). \nonumber 
\eeq

Given the expression for $\overline{\gamma}$,
\beq 
\overline{\alpha}/\underline{\alpha}\overline{\gamma} = \overline{\alpha}/\underline{\alpha} - 1 + \gamma. \nonumber 
\eeq 

Let $\delta_1 > 0$ be such that $\overline{\alpha}/\underline{\alpha} = 2 - \delta_1$. Then, there exists a $M > 0$ sufficiently large such that $\gamma = \frac{\overline{\alpha}^2}{\underline{\alpha}^2(M-1)} = \delta_2 < \delta_1$. Let $\delta_3 = \delta_1 - \delta_2$. Then,
\beq 
\overline{\alpha}/\underline{\alpha}\overline{\gamma} = \overline{\alpha}/\underline{\alpha} - 1 + \gamma = 1 - \delta_3 < 1. \nonumber
\eeq 

Let $\overline{\gamma}_1 = 1 - \delta_3$. Then, for every $b > 0$ there exists $T_0 > 0$ sufficiently large and a constant $\tilde{b} \geq b$ such that 
\begin{align}
& \sigma(\hat{x}_{t+1}) \leq \overline{\gamma}_1\sigma(\hat{x}_t) + \mathcal{O}(\frac{g(\tilde{b},M)}{\overline{\alpha}\sqrt{T_0}}) \nonumber \\
& \leq   \sigma(\hat{x}_t) \leq \max_{x: \norm{x}\leq \tilde{b}} \sigma(x) ~ \Rightarrow ~ \norm{\hat{x}_{t+1}}\leq \tilde{b}. \nonumber  
\end{align}
 
Now, we can choose $b$ to be such that $b = \norm{y_{T_0+1}}$. This $b$ is a constant that does not increase with $T_0$ because $\norm{y_{T_0+1}}$ is bounded by a constant that does not increase with $T_0$. This is because $\rho(A^\star) < 1$ and $\norm{u^{\pi_\mathrm{o}}_t} \leq (n+1)m$ in the estimation phase. For this $b$, we can pick $M$, $T_0$ and $\tilde{b}$ as above. Since $\sigma(x)$ is a continuous function, $\tilde{b}$ is a constant that does not increase with $T_0$ and $T$. 

Then, by mathematical induction, it follows that under the event $\theta^\star \in \widehat{\Theta}$, $\norm{\hat{x}_t} \leq \tilde{b}$ for all $t \geq T_0+1$, where $\tilde{b}$ is a constant that does not increase with $T_0$ and $T$. 

Consequently, Eq. \eqref{eq:Vt-bound-o-rhc} is true for all $t \geq T_0+1$, under the event $\theta^\star \in \widehat{\Theta}$. Therefore, applying the inequality in Eq. \eqref{eq:Vt-bound-o-rhc} repeatedly, we get under the event $\theta^\star \in \widehat{\Theta}$
\begin{align}
& V_{t}(\hat{x}_{t};\hat{\theta}_{t}) \leq \overline{\gamma}^{t-T_0-1}V_{T_0+1}(\hat{x}_{T_0+1};\hat{\theta}_{T_0+1}) \nonumber \\
& + \mathcal{O}(\frac{1}{\sqrt{T_0}})\sum_{k=T_0+1}^t \overline{\gamma}^{k-T_0-1}  = \overline{\gamma}^{t-T_0-1}V_{T_0+1}(\hat{x}_{T_0+1};\hat{\theta}_{T_0+1}) \nonumber \\
& + \mathcal{O}(\frac{1}{\sqrt{T_0}}),  \nonumber 
\end{align} 
where the last equality follows from the fast that $\overline{\gamma} < 1$. Using $c_t(\hat{x}_t, u^{\pi_{\mathrm{o}}}_{t}) \leq V_{t}(\hat{x}_{t};\hat{\theta}_{t})$, and summing over all $t \geq T_0+1$, we get
\begin{align} 
& \sum_{t = T_0+1}^T c_t(\hat{x}_t, u^{\pi_{\mathrm{o}}}_{t}) \leq \sum_{t = T_0+1}^T \overline{\gamma}^{t-T_0-1}V_{T_0+1}(\hat{x}_{T_0+1};\hat{\theta}_{T_0+1})\nonumber \\
& + \sum_{t = T_0+1}^T \mathcal{O}(\frac{1}{\sqrt{T_0}}) \stackrel{(r)}{=} \mathcal{O}(V_{T_0+1}(\hat{x}_{T_0+1};\hat{\theta}_{T_0+1})) \nonumber \\
& + \mathcal{O}(\frac{T}{\sqrt{T_0}}) \stackrel{(s)}{=} \mathcal{O}(\frac{T}{\sqrt{T_0}}). \nonumber 
\end{align}

Here, we get (r) by the fact that $\overline{\gamma} < 1$ and (s) by the fact that $\hat{\theta}_{T_0+1}$ is optimistic and by Assumption \ref{as:stability}.$(ii)$, both of which imply $V_{T_0+1}(\hat{x}_{T_0+1};\hat{\theta}_{T_0+1}) \leq \overline{\alpha} \sigma(\hat{x}_{T_0+1}) = \mathcal{O}(1)$. This also implies that, with probability greater than $1-\delta$, $J_{T_{0}+1:T}(u^{\pi_{\mathrm{o}}}_{T_{0}+1:T}; \hat{\theta}_{T_0+1:T}) -  J_{T_{0}+1:T}(u^{\pi^{\star}}_{T_{0}+1:T};\theta^{\star}) \leq  \mathcal{O}(\frac{T}{\sqrt{T_0}})$, which concludes the proof.
\end{proof}

\end{appendices} 


\end{document}